\documentclass{article} 
\usepackage{iclr2023_conference,times}

\usepackage{hyperref}
\usepackage{url}

\usepackage{amssymb, amsmath, amsthm}

\definecolor{shadecolor}{gray}{0.9}

\usepackage{amsmath}
\usepackage{enumitem}
\usepackage{amssymb}
\newcommand{\mathbold}[1]{\ensuremath{\boldsymbol{\mathbf{#1}}}}
\DeclareRobustCommand{\mb}[1]{\mathbold{#1}}
\usepackage{amsthm}

\usepackage{graphicx}
\usepackage[labelfont=bf]{caption}
\usepackage[format=hang]{subcaption}

\usepackage{booktabs}

\usepackage[algoruled, algo2e]{algorithm2e}
\usepackage{algorithm}
\usepackage{algorithmic}

\usepackage{listings}
\usepackage{fancyvrb}

\usepackage{natbib}

\usepackage{xcolor}
\hypersetup{
    colorlinks,
    linkcolor={red!50!black},
    citecolor={blue!50!black},
    urlcolor={blue!80!black}
}

\usepackage[nameinlink]{cleveref}

\usepackage[acronym,smallcaps,nowarn]{glossaries}
\glsdisablehyper

\usepackage{nccmath}

\usepackage[theorems,skins]{tcolorbox}

\usepackage{wrapfig}

\usepackage{nicefrac}

\usepackage{multirow}
\usepackage{soul}
\usepackage{color}
\newacronym{mcmc}{mcmc}{Markov chain Monte Carlo}

\newacronym{ddpm}{ddpm}{denoising diffusion probabilistic model}
\newacronym{sde}{sde}{stochastic differential equation}
\newacronym{elbo}{elbo}{evidence lower bound}
\newacronym{nhl}{nhl}{Nose-Hoover Langevin}
\newacronym{kl}{kl}{Kullback-Leibler}

\newacronym{bpd}{bpd}{bits-per-dim}
\newacronym{nfe}{nfe}{number of function evaluations}

\newacronym{vae}{vae}{variational autoencoder}

\newacronym{amdt}{amdt}{Automatic Multivariate Diffusion Training}
\newacronym{ode}{ode}{ordinary differential equation}
\newacronym{dbgm}{dbgm}{diffusion-based generative model}

\newacronym{cld}{cld}{critically-damped langevin diffusion}
\newacronym{vpsde}{vpsde}{variance-preserving stochastic differential equation}
\newacronym{vesde}{vesde}{variance-exploding stochastic differential equation}
\newacronym{alda}{alda}{accelerated Langevin diffusion}
\newacronym{malda}{malda}{modified accelerated Langevin diffusion}

\newacronym{mle}{mle}{maximum likelihood estimation}

\newacronym{dsm}{dsm}{Denoising Score Matching}
\newacronym{ism}{ism}{Implicit Score Matching}

\newacronym{iwae}{iwae}{importance-weighted auto-encoder}

\newacronym{vp}{vp}{variance preserving}
\newacronym{ve}{ve}{variance exploding}

\newacronym{mdm}{mdm}{Multivariate Diffusion Model}

\newacronym{pfode}{pfode}{\textit{probability flow} \gls{ode} }

\newcommand{\var}[1]{\text{Var}\left(#1\right)}
\newcommand{\norm}[1]{\left\lVert#1\right\rVert}

\def\eqref#1{equation~\ref{#1}}

\def\1{\bm{1}}

\def\eps{{\epsilon}}

\newcommand{\mba}{\mathbold{a}}

\newcommand{\mbm}{\mathbold{m}}

\newcommand{\mbu}{\mathbold{u}}
\newcommand{\mbv}{\mathbold{v}}

\newcommand{\mbx}{\mathbold{x}}
\newcommand{\mby}{\mathbold{y}}
\newcommand{\mbz}{\mathbold{z}}

\newcommand{\mbA}{\mathbold{A}}
\newcommand{\mbB}{\mathbold{B}}
\newcommand{\mbC}{\mathbold{C}}
\newcommand{\mbD}{\mathbold{D}}

\newcommand{\mbH}{\mathbold{H}}
\newcommand{\mbI}{\mathbold{I}}

\newcommand{\mbL}{\mathbold{L}}
\newcommand{\mbM}{\mathbold{M}}
\newcommand{\mbN}{\mathbold{N}}

\newcommand{\mbQ}{\mathbold{Q}}
\newcommand{\mbR}{\mathbold{R}}
\newcommand{\mbS}{\mathbold{S}}

\newcommand{\mbV}{\mathbold{V}}

\newcommand{\mbZ}{\mathbold{Z}}

\def\rvm{{\mathbf{m}}}

\def\rvs{{\mathbf{s}}}

\def\rvu{{\mathbf{u}}}
\def\rvv{{\mathbf{v}}}

\def\rvx{{\mathbf{x}}}
\def\rvy{{\mathbf{y}}}
\def\rvz{{\mathbf{z}}}

\def\rmA{{\mathbf{A}}}
\def\rmB{{\mathbf{B}}}
\def\rmC{{\mathbf{C}}}
\def\rmD{{\mathbf{D}}}

\def\rmH{{\mathbf{H}}}
\def\rmI{{\mathbf{I}}}

\def\rmL{{\mathbf{L}}}

\def\rmQ{{\mathbf{Q}}}

\def\rmU{{\mathbf{U}}}

\def\rmZ{{\mathbf{Z}}}

\DeclareMathAlphabet{\mathsfit}{\encodingdefault}{\sfdefault}{m}{sl}
\SetMathAlphabet{\mathsfit}{bold}{\encodingdefault}{\sfdefault}{bx}{n}

\def\gN{{\mathcal{N}}}

\def\sR{{\mathbb{R}}}

\theoremstyle{plain}
\newtheorem{theorem}{Theorem}
\theoremstyle{definition}
\newtheorem*{assumption*}{Assumption}

\newtheorem*{proposition*}{Proposition}

\newcommand{\E}{\mathbb{E}}

\title{Where to Diffuse, How to Diffuse, and How to Get Back: Automated Learning for Multivariate Diffusions}

\author{Raghav Singhal$^{*,1}$,
Mark Goldstein$^{* 1}$,  Rajesh Ranganath$^{1,2}$
 \\
Courant Institute of Mathematical Sciences$^1$, New York University\\
Center for Data Science$^2$, New York University
}

\newcommand\blfootnote[1]{%
  \begingroup
  \renewcommand\thefootnote{}\footnote{#1}%
  \addtocounter{footnote}{-1}%
  \endgroup
}

\iclrfinalcopy 
\begin{document}

\maketitle

\begin{abstract}
\Glspl{dbgm} perturb data to a target noise distribution and reverse this process to generate samples. The choice
of noising process, or \textit{inference diffusion process}, affects both likelihoods and sample quality.
For example, extending the inference process with auxiliary variables leads to improved sample quality.
While there are many such multivariate diffusions to explore, each new one requires significant model-specific analysis, hindering rapid prototyping and evaluation. In this work, we study \glspl{mdm}. For any number of auxiliary variables, we provide a recipe for maximizing a lower-bound on the \glspl{mdm} likelihood without requiring any model-specific analysis. We then demonstrate how to parameterize the diffusion for a specified target noise distribution; these two points together enable optimizing the inference diffusion process. Optimizing the diffusion expands easy experimentation from just a few well-known processes to an automatic search over all linear diffusions. To demonstrate these ideas, we introduce two new specific diffusions as well as learn a diffusion process on the \textsc{mnist}, \textsc{cifar}10, and \textsc{imagenet}32 datasets. 
We show learned \glspl{mdm} match or surpass \glspl{bpd} relative to fixed choices of diffusions for a given dataset and model architecture.
\end{abstract}

\glsresetall
\section{Introduction}
\Glspl{dbgm} perturb data to a target noise distribution and reverse this 
process to generate samples. 
They have achieved impressive performance in image generation, editing, translation 
\citep{dhariwal2021diffusion, nichol2021improved, sasaki2021unit, ho2022cascaded},
conditional text-to-image tasks 
\citep{nichol2021glide, ramesh2022hierarchical, saharia2022photorealistic} and music and audio generation
\citep{chen2020wavegrad,kong2020diffwave,mittal2021symbolic}. They are often trained by maximizing a lower bound on the log likelihood, featuring an inference process interpreted as gradually ``noising" the data
\citep{sohl2015deep, ho2020denoising}.
\blfootnote{$^*$ Equal Contribution. Correspondence to \texttt{\{rsinghal,goldstein\} at nyu.edu}.}

The choice of this
inference process affects both likelihoods and sample quality.
On different datasets and models, different inference processes work better; there is no universal best choice of inference,
and the choice matters \citep{song2020score}.

While some work has improved performance by designing score model architectures \citep{ho2020denoising,kingma2021variational,dhariwal2021diffusion},  \citet{dockhorn2021score} instead introduce the \gls{cld}, showing that significant improvements in sample generation can be gained by carefully designing new processes. 
\gls{cld} pairs each data dimension with an auxiliary ``velocity" variable and diffuses them jointly using second-order Langevin dynamics.

A natural question: if introducing new diffusions results in dramatic performance gains, why are there only a handful of diffusions (\gls{vpsde}, \gls{ve}, \gls{cld}, sub-\gls{vpsde}) used in \glspl{dbgm}? For instance, are there other auxiliary variable diffusions that would lead to improvements like \gls{cld}? This avenue seems promising as auxiliary variables have improved other generative models and inferences, such as normalizing flows \citep{huang2020augmented}, neural \glspl{ode} \citep{dupont2019augmented}, hierarchical variational models \citep{ranganath2016hierarchical},
ladder variational autoencoder \citep{sonderby2016ladder}, among others.

Despite its success, \gls{cld} also provides evidence that each new process requires significant model-specific analysis.
Deriving the \gls{elbo} and training algorithm for diffusions is challenging 
\citep{huang2021variational,kingma2021variational,song2021maximum} and is carried out in a case-by-case manner for new diffusions \citep{campbell2022continuous}. Auxiliary variables seemingly complicate this process further; computing
conditionals of the inference process
necessitates solving matrix Lyupanov equations
(\cref{sec:transition}).
Deriving the inference stationary distribution---which helps the model and inference match---can be intractable. These challenges
limit rapid prototyping and evaluation of new 
inference processes.

Concretely, training a diffusion model requires:
\begin{enumerate}[start=1,label={(\bfseries R\arabic*):}]

\item Selecting an inference and model process pair
such that the inference process converges
to the model prior
\item Deriving the \gls{elbo} for this pair
\item Estimating the \gls{elbo} and its gradients
by deriving and computing the inference process'
transition kernel
\end{enumerate}

In this work, we introduce \glspl{mdm} and a method for training  and evaluating them.
\glspl{mdm} are
\acrlongpl{dbgm}
trained with auxiliary variables.
We provide a recipe
for training \glspl{mdm} beyond specific instantiations--like \gls{vpsde} and \gls{cld}---to all linear inference processes
that have a stationary distribution,
with any number of auxiliary variables.

First, we bring results from gradient-based MCMC \citep{ma2015complete} to diffusion modeling to 
construct \glspl{mdm} that converge to a chosen model prior (\textbf{R1}); this tightens the \gls{elbo}.
Secondly, for any number of auxiliary variables, we derive the \gls{mdm} \gls{elbo} (\textbf{R2}).
Finally,
we show that the transition kernel of linear \glspl{mdm},
necessary for the \gls{elbo}, can be computed automatically and generically, for higher-dimensional auxiliary systems (\textbf{R3}).

With these tools, we explore
a variety of 
new inference processes for \acrlongpl{dbgm}.
We then note that the automatic transitions and fixed stationary distributions facilitate directly learning the inference to maximize the \gls{mdm} \gls{elbo}. Learning turns diffusion model training into a search not only over score
models but also inference processes, at no extra derivational cost.

\paragraph{Methodological Contributions.} In summary, our methodological contributions are:
\begin{enumerate}

 \item Deriving \glspl{elbo} for training and evaluating multivariate diffusion models
(\glspl{mdm}) with auxiliary variables.

    \item Showing that the diffusion transition covariance does not need to be manually derived for each new diffusion. We instead demonstrate that a matrix factorization technique, previously unused in diffusion models, can automatically compute the covariance analytically for any linear \gls{mdm}.
    
    \item Using results from gradient-based \gls{mcmc}
    to construct \glspl{mdm} with a complete parameterization of inference processes whose stationary distribution matches the model prior.
    
    \item Combining the above into an algorithm
    called \gls{amdt}
    that enables training without diffusion-specific
    derivations. \gls{amdt} enables training score models for 
    any linear diffusion, including optimizing the diffusion and score jointly.
\end{enumerate}
To demonstrate these ideas, we develop \glspl{mdm} with two specific diffusions as well as learned multivariate diffusions. The specific diffusions are \gls{alda} (introduced in  \citet{mou2019high} as a higher-order scheme for gradient-based \gls{mcmc}) and an alteration, \gls{malda}. Previously, using these diffusions for generative modeling would require significant model-specific analysis. Instead, \gls{amdt} for these diffusions is derivation-free.

\paragraph{Empirical contributions.} 

We train \glspl{mdm} on the \textsc{mnist}, \textsc{imagenet}32 and \textsc{cifar}-10 datasets. In the experiments, we show that:
\begin{enumerate}

\item Training new and existing fixed diffusions, such as \gls{alda}
and \gls{malda}, is easy with the proposed algorithm \gls{amdt}.

\item Using \gls{amdt} to learn the choice of diffusion for the \gls{mdm}
matches or surpasses the performance of fixed choices of diffusion process; sometimes the learned diffusion and \gls{vpsde} do best; other times the learned diffusion and \gls{cld} do best.

\item  There are new and existing \glspl{mdm}, trained and evaluated with the \gls{mdm} \gls{elbo}, that account for as much performance improvement
over \gls{vpsde} as a three-fold increase in score model size for a fixed univariate diffusion.
\end{enumerate}

These findings affirm that the choice of diffusion affects the optimization problem, and that learning the choice bypasses the process of choosing diffusions for each new dataset and score architecture.
We additionally show the utility of the \gls{mdm} \gls{elbo} by showing on a dataset that \gls{cld} achieves better \glspl{bpd} than previously reported with the probability flow \textsc{ode} \citep{dockhorn2021score}.

\section{Setup}

We present diffusions by starting with the generative model
and then describing its likelihood lower bound
\citep{sohl2015deep,huang2021variational,kingma2021variational}. Diffusions sample from a model prior $\mbz_0 \sim \pi_\theta$
and then evolve a continuous-time stochastic process
$\mbz_t \in \mathbb{R}^d$:
\begin{align}
    \label{eq:model}
    d\mbz = h_\theta(\mbz,t)dt+\beta_\theta(t)d\mbB_t, \quad t \in [0, T]
\end{align}
where $\mbB_t$ is a $d$-dimensionsal Brownian motion. The model is trained so that $\mbz_T$ approximates the data $\mbx \sim q_{\text{data}}$.\footnote{Following \cite{huang2021variational,dockhorn2021score}
we integrate all processes in forward time $0$ to $T$. It may be helpful to think of an additional variable $\hat{\mbx}_{t} \triangleq \mbz_{T-t}$ so that $\hat{\mbx}_0$ approximates $\mbx \sim q_{\text{data}}$.} Maximum likelihood training of diffusion models is intractable
\citep{huang2021variational,song2021maximum,kingma2021variational}. Instead, 
they are trained using a variational lower bound on $\log p_{\theta}(\rvz_T=x)$. 
The bound requires an inference process $q_\phi(\mby_s|\mbx=x)$:\footnote{We use $\mby$ as the inference variable over the same space as the model's $\mbz$.}
\begin{align}
d\mby = f_\phi(\mby,s)ds + g_\phi(s)d\widehat{\mbB}_s, \quad s \in [0, T]
\end{align}
where $\widehat{\mbB}_s$ is another Brownian motion independent of $\mbB_t$.
The inference process is usually
taken to be specified rather than learned,
and chosen to be i.i.d.
for each $y_{tj}$ conditional on each
$x_j$. This leads to the interpretation of the $y_{tj}$ as noisy versions of features $x_{j}$ \citep{ho2020denoising}.
While the diffusion \gls{elbo} is challenging to derive in general, \cite{huang2021variational,song2021maximum}
show that when the model process takes the form:
\begin{align}\label{eq:model_parameterization}
    d\rvz = \left[g_\phi^2(T-t)s_\theta(\rvz,T-t)-f_\phi(\rvz,T-t)\right] dt + g_\phi(T - t) d\rmB_t ,
\end{align}
the \gls{elbo} is:
\begin{align}
    \label{eq:ism_bound}
\begin{split}
\log p_\theta(x)\geq \mathcal{L}^{\text{ism}}(x) =
\E_{q_\phi(\mby|x)} \Bigg[ \log\pi_\theta(\mby_T) +
\int_0^T -\frac{1}{2} \|s_\theta\|^2_{g^2_\phi} - \nabla \cdot (g_\phi^2 s_\theta - f_\phi) ds
\Bigg],
\end{split}
\end{align}
where  $f_\phi,g_\phi,s_\theta$ are evaluated at $(\mby_s,s)$, $\|\rvx\|^2_\rmA=\rvx^\top \rmA \rvx$ and $g^2=g g^\top$.  \Cref{eq:ism_bound} features the \gls{ism} loss \citep{song2020sliced},
and can be re-written as an \gls{elbo} 
$\mathcal{L}^{\text{dsm}}$ featuring \gls{dsm} \citep{vincent2011connection,song2020score}, see
\cref{appsec:ism_to_dsm}. 

\section{A recipe for Multivariate Diffusion Models \label{sec:mdm}}

As has been shown in prior work \citep{song2021maximum,dockhorn2021score}, the choice of diffusion matters. Drawing on principles from
previous
generative models (\cref{sec:related}), we
can consider a wide class
of diffusion inference processes
by constructing them using auxiliary variables.

At first glance, training such diffusions can seem challenging. First, one needs an \gls{elbo} that includes auxiliary variables. This \gls{elbo} will require 
sampling from the transition kernel,
and setting the model prior to the specified inference stationary distribution. But doing such diffusion-specific analysis manually
is challenging and hinders rapid prototyping. 

In this section we show how to address these challenges and introduce an algorithm, \gls{amdt},
to simplify and automate modeling with \glspl{mdm}. \gls{amdt} can
be used to train new and existing diffusions, including those with auxiliary variables, and including
those that learn the inference process. In \cref{appsec:score_matching} we discuss how the presented methods can also be used to automate and improve simplified score matching and noise prediction objectives used to train diffusion models.

\subsection{Multivariate Model and Inference \label{sec:multivar_model_and_inference}}
For the $j^{th}$ data coordinate at each time $t$,
\glspl{mdm} pair $\mbz_{tj} \in \mathbb{R}$ with a vector of auxiliary variables $\mbv_{tj} \in \mathbb{R}^{K-1}$ into a joint vector $\mbu_t$
and diffuse in the extended space:
  \begin{align}
        \label{eq:multivar_model}
        \mbu_0 \sim \pi_\theta,
        \quad \quad 
        d\mbu = h_\theta
        (
         \mbu_t = \begin{bmatrix}
      \mbz_t \\
      \mbv_t
    \end{bmatrix},
        t) 
        dt + \beta_\theta(t) d\mbB_t.
    \end{align}
\glspl{mdm} model the data $\mbx$ with $\mbz_T$, a coordinate in $\mbu_T \sim p_\theta$.
For the $j^{th}$ feature $\mbx_j$, each $\mbu_{tj} \in \mathbb{R}^K$  consists
of a ``data" dimension 
$\mbu_{tj}^z$ and auxiliary variable
$\mbu_{tj}^v$.
Therefore $\mbu \in \mathbb{R}^{dK}$. 
We extend
the drift coefficient $h_\theta$ from
a function in $\mathbb{R}^d \times \mathbb{R}_+ \rightarrow \mathbb{R}^d$ to
the extended space $\mathbb{R}^{dK} \times \mathbb{R}_+ \rightarrow \mathbb{R}^{dK}$. We likewise
extend the diffusion coefficient to a matrix $\beta_\theta$ acting on Brownian motion $\mbB_t \in \mathbb{R}^{dK}$. 

Because the \gls{mdm} model is over the extended space, the inference distribution $\mby$ must be too.
We then set $q(\mby_0^v|\mby_0^z=x)$ to
any chosen initial distribution,
e.g. $\mathcal{N}(\mathbold{0},\mbI)$
and discuss this choice in \cref{sec:insights}. Then
$\mby_s$ evolves according to 
the auxiliary variable inference process:
\begin{align}
\label{eq:multivar_inference}
 d\mby = f_\phi(\mby,s)ds + g_\phi(s) d \widehat{\mbB}_s,   
\end{align}
where the inference drift and diffusion coefficients $f_\phi, g_\phi$ are now over the extended space $\mby=[\mby^z,\mby^v]$. The function $f_\phi$
lets the $z$ and $v$ coordinates of $\mby_{tj}$ interact in the inference process.

\subsection*{Assumptions}

This work demonstrates how to parameterize 
time-varying It\^o processes, used for diffusion modeling, to have a stationary distribution that matches the given model prior. To take advantage of the automatic transition kernels also presented, the inferences considered for modeling are linear time-varying processes and take the form:
\begin{align*}
    d\rvy = \rmA_{\phi}(s) \rvy ds + g_\phi(s)d\rmB_s
\end{align*}
where $\rmA_\phi(s): \mathbb{R}_+ \rightarrow dK \times dK$ and $g_\phi(s): \mathbb{R}_+ \rightarrow dK \times dK$ are matrix-valued functions.

\subsection{\gls{elbo}  for \glspl{mdm} }
We now show how to train \glspl{mdm} to optimize a lower bound on the log likelihood of the data. 
Like in the univariate case, 
we use the parameterization in \cref{eq:model_parameterization} to obtain a tractable \gls{elbo}.
\begin{theorem}
The \gls{mdm} log marginal likelihood of the data is lower-bounded by:
\begin{align}
\begin{split}
    \label{eq:mdm_bound}
\log & p_\theta(x) \geq
\E_{
q_\phi(\mby|x)
}
    \Bigg[ 
    \underbrace{\log\pi_\theta(\mby_T)}_{\ell_T}
    -
\int_0^T
    \frac{1}{2} \|s_\theta\|^2_{g^2_\phi}
    +
    \nabla \cdot (g_\phi^2 s_\theta - f_\phi)
    ds
     -
    \underbrace{\log q_\phi(\mby_0^v|x)}_{\ell_q}
    \Bigg] 
    \quad (\mathcal{L}^{\text{mism}})
    \\
&=
\E_{
q_\phi(\mby|x)
}
    \Bigg[ 
    \ell_T +
\int_0^T
\frac{1}{2}\|
    s_{\phi}
\| ^2_{g^2_\phi}
    -  \frac{1}{2} \|
    s_\theta
    - 
    s_{\phi}
    \|
    ^2_{g^2_\phi} 
    +
        (\nabla \cdot f_\phi)ds
    -
   \ell_q
    \Bigg] \quad (\mathcal{L}^{\text{mdsm}}).
    \end{split}
\end{align} 
where divergences and gradients are taken with respect to
$\mby_s$ and $s_{\phi}=\nabla_{\mby_s} \log q_\phi(\mby_s|x)$.
\end{theorem}
\begin{proof}
The proof for the \gls{mdm} \gls{ism} \gls{elbo} $\mathcal{L}^{\text{mism}}$
 is in \cref{appsec:mdm_elbo}. In short, 
we introduce auxiliary variables, 
apply Theorem 1 of \cite{huang2021variational} (equivalently, Theorem 3 of \cite{song2021maximum} or appendix E of \cite{kingma2021variational})
to the joint space, and then apply
an additional variational bound to $\rvv_0$. The 
\gls{mdm} \gls{dsm} \gls{elbo} $\mathcal{L}^{\text{mdsm}}$
is likewise derived in 
 \cref{appsec:mdm_elbo}, similarly to
\cite{huang2021variational,song2021maximum}, but
extended to multivariate diffusions.
\end{proof}

We train \gls{mdm}'s by estimating the gradients of
$\mathcal{L}^{\text{mdsm}}$,
as estimates of $\mathcal{L}^{\text{mism}}$ can be computationally prohibitive. 
For numerical stability, the integral in \cref{eq:mdm_bound} is computed on $[\epsilon, T]$ rather than $[0, T]$.
One can regard this as a bound for a variable $\mbu_{\epsilon}$.
To maintain a proper likelihood bound for the data, one can choose a likelihood $\mbu_0 | \mbu_{\epsilon}$ and compose bounds as we demonstrate in \cref{appsec:offset_math}. We report the \gls{elbo} with this likelihood term, which plays the same role as the discretized Gaussian in \cite{nichol2021improved} and Tweedie's formula in \cite{song2021maximum}.

\subsection{Ingredient 1: Computing the transition  $q_\phi(\mby_s|x)$ \label{sec:transition}}
To estimate \cref{eq:mdm_bound} and its gradients,
we need samples from
$q(\mby_s|x)$ and to compute
$\nabla \log q(\mby_s|x)$.
While an intractable problem for \glspl{mdm} in general,
we provide two ingredients for tightening and optimizing these bounds in a generic fashion for linear inference \glspl{mdm}.

We first show how to automate computation of $q(\mby_s|\mby_0)$ and then $q(\mby_s|x)$. For linear \glspl{mdm} of the form: 
\begin{align*}
    d\rvy = \rmA(s) \rvy ds + g(s)d\rmB_s ,
\end{align*}
the transition kernel $q(\mby_s|\mby_0)$ is Gaussian \citep{sarkka2019applied}. Let $f(\mby,s)=\mbA(s)\mby$. Then, the mean and covariance
are solutions to the following \glspl{ode}:
\begin{align}
    d\mbm_{s|0}/ds &= \mbA(s)\mbm_{s|0} \nonumber \\ \label{eq:ode_cov}
    d \mathbold{\Sigma}_{s|0}/ds &= \mbA(s) \mathbold{\Sigma}_{s|0} + \mathbold{\Sigma}_{s|0} \mbA^\top(s) + g^2(s)  .
\end{align}
The mean can be solved analytically:
\begin{align}\label{eq:analytic_mean}
\mbm_{s|0}=\exp \left[\int_0^s \mbA(\nu) d\nu \right] \mby_0 \underbrace{=\exp(s \mbA)\mby_0}_{\text{no integration if $\mbA(\nu)=\mbA$}} .
\end{align}
The covariance equation does not have as simple a solution because \cref{eq:analytic_mean} as the unknown matrix $\mathbold{\Sigma}_{s|0}$ is being multiplied both from the left and the right. 

Instead of solving \cref{eq:ode_cov} for a specific diffusion \textit{manually}, as done in 
previous work (e.g. pages 50-54 of \cite{dockhorn2021score}),  we show that a matrix 
factorization technique (\citet{sarkka2019applied}, sec. $6.3$)
previously unused in \acrlongpl{dbgm} can automatically compute 
$\mathbold{\Sigma}_{s|0}$ generically for any linear \gls{mdm}.
Define $\mbC_s, {\mbH_s}$ that evolve according to:
\begin{align}
  \begin{pmatrix}
    d \mbC_s /ds \\
    d \mbH_s / ds
  \end{pmatrix} =
  \begin{pmatrix}
    \mbA(s) & g^2(s) \\
    \mathbold{0} & -\mbA^\top(s)
  \end{pmatrix}
        \begin{pmatrix}
          \mbC_s \\
          \mbH_s
        \end{pmatrix},
\end{align}
then $\mathbold{\Sigma}_{s|0}=\mbC_s{\mbH_s}^{-1}$ 
for $\mbC_0=\mathbold{\Sigma}_0$ and $\mbH_0=\mbI$
(\Cref{appsec:kernel}).
These equations can be solved in closed-form,
    \begin{align}
      \begin{pmatrix}
        \mbC_s \\
        \mbH_s
      \end{pmatrix} =
      \exp\Bigg[
      \begin{pmatrix}
        [\mbA]_s & [g^2]_s \\
        \mathbold{0} &  -[\mbA^\top]_s
      \end{pmatrix}
      \Bigg]
            \begin{pmatrix}
              \mathbold{\Sigma}_0 \\
              \mathbold{I} 
            \end{pmatrix} 
            \underbrace{=
    \exp\Bigg[ s
      \begin{pmatrix}
        \mbA & g^2 \\
        \mathbold{0} & -\mbA^{\top}
      \end{pmatrix}
      \Bigg]}_{\text{no integration if $\mbA(\nu)=\mbA,g(\nu)=g$}}
            \begin{pmatrix}
              \mathbold{\Sigma_0} \\
              \mathbold{I}
            \end{pmatrix},
    \end{align}
where $[\mbA]_s = \int_0^s \mbA(\nu)d\nu$. 
To condition on $\mby_0=(x,v)$, we set $\mathbold{\Sigma}_0=\mathbold{0}$. 

\paragraph{Computing $q_\phi(\mby_s|x)$.}

For the covariance $\mathbold{\Sigma}_{s|0}$, to condition on $x$ instead of $\mby_0$, we set $\mathbold{\Sigma_0}$ to 
\begin{align*}
    \mathbold{\Sigma_0} = \begin{pmatrix}
    0 & 0 \\
    0 & \mathbold{\Sigma_{v_0}}
    \end{pmatrix} ,
\end{align*}
To compute the mean, it is the same expression
as for $q(\mby_s|\mby_0)$, but with a different initial condition:
\begin{align}
\mbm_{s|0}=\exp \left[\int_0^s \mbA(\nu) d\nu \right] \begin{pmatrix}
x\\
\E_q[\mby^v_0|x]
\end{pmatrix}
\end{align}
See \cref{appsec:kernel} for more details.

\begin{wraptable}{r}{5.5cm}
\caption{\textbf{Runtime Comparison}: we compare the run time of sampling from the \gls{cld} diffusion analytically versus using the automated algorithm.}\label{wrap-tab:runtime}
    \begin{tabular}{lll}
        Method &  \textsc{cifar}-10  & \textsc{mnist} \\ \hline
        Analytical &  0.027 & 0.0062 \\
        Automated  & 0.029 & 0.007 \\
    \end{tabular}
\end{wraptable}

\paragraph{A fast and simple algorithm.}
We show in \cref{alg:transition_generic} (\cref{appsec:algo}) that computing the transition kernel only requires knowing $f,g$ and requires no diffusion-specific analysis.
For $K-1$ auxiliary variables, $\mbA , g$  are $K \times K$. Like for scalar diffusions, these parameters are shared across data coordinates. This means matrix exponentials and inverses are done on $K \times K$ matrices, where $K$ is only $2$ or $3$ in our experiments. In \cref{wrap-tab:runtime}, we compare the time to sample a batch of size $256$ from the transition kernel for \textsc{cifar10} and \textsc{mnist}. The table shows the extra computational cost of the automated algorithm is negligible. This automation likewise applies to simplified score matching 
and noise prediction objectives, since all rely on $q_\phi(\mby_s | x)$
(\cref{appsec:score_matching}).

\begin{algorithm}[t!]
\begin{algorithmic}
 \STATE {\bfseries Input:}
    Data $\{x_i\}$, inference process matrices $\rmQ_\phi, \rmD_\phi$, model prior $\pi_{\theta}$, initial distribution $q_{\phi}(\rvy^v_0 \mid x)$, and score model architecture $s_{\theta}$  
     \STATE {\bfseries Returns:} Trained score model $s_{\theta}$
    \WHILE{$s_\theta$ not converged}
        \STATE Sample $x \sim \sum_{i=1}^N \frac{1}{N} \delta_{x_i}$, $v_0 \sim q_{\phi}(\rvy^v_0 \mid x)$
        \STATE Sample $\rvs \sim \rmU[0, T]$ and $\rvy_s, \rvy_T \sim q_\phi(\rvy_s \mid x)$ using \cref{alg:transition_generic} 
        \STATE Estimate the stochastic gradient of the \gls{mdm} \gls{elbo}, $\nabla_\theta \mathcal{L}(\theta,\phi)$, using \cref{eq:mdm_bound}
        \STATE $\theta \leftarrow \theta + \alpha \nabla_{\theta}\mathcal{L}(\theta,\phi)$.
        \IF {learning inference}
            \STATE  $\phi \leftarrow \phi + \alpha \nabla_\phi \mathcal{L}(\theta, \phi)$
        \ENDIF
    \ENDWHILE
    \STATE{ \bfseries Output} $s_{\theta}$
\end{algorithmic}
\vskip -0.05in
\caption{\label{alg:amdt} Automatic Multivariate Diffusion Training}
\end{algorithm}

\subsection{Ingredient 2: \gls{mdm} Parameterization}

The \gls{mdm} \gls{elbo} (\cref{eq:mdm_bound}) is tighter when the inference $\mby_T$ tends toward the model's prior $\pi_\theta$. Here we construct inference processes with the model prior $\pi_\theta$ as a specified stationary distribution $q_{\infty}$. 

\cite{ma2015complete} provide a complete recipe for constructing gradient-based \textsc{mcmc} samplers; the recipe
constructs non-linear time-homogeneous It\^o processes with a given stationary distribution, and show that the parameterization spans all 
such It\^o processes with that stationary distribution.

Diffusion models usually have time-varying drift and diffusion coefficients (e.g. use of the $\beta(t)$ function).
To build diffusion models
that match the model prior, we
first extend Theorem 1 from \cite{ma2015complete}
to construct non-linear It\^o processes with time-varying drift and diffusion coefficients with a given stationary 
distribution (\Cref{appsec:stationary}). 
Then, to keep transitions tractable (per \Cref{sec:transition}), we specialize this result to linear It\^o diffusions. 

We directly state the result for linear time-varying diffusions with stationary distributions.
The parameterization requires a skew-symmetric matrix $-\mbQ(s)=\mbQ(s)^\top$, a positive semi-definite matrix $\mbD(s)$, and 
a function $\nabla H(\mby)$ such that the desired stationary distribution $q_{\infty}$ is proportional to $\exp[-H(\mby)]$.  Linear It\^o diffusions have Gaussian stationary distributions \citep{sarkka2019applied} meaning that $\nabla H$ is linear and can be expressed as $\mbS \mby$ for some matrix $\mbS$. 
For a matrix $\mbA$, let $\sqrt{\mbA}$ refer to the matrix square root defined by $\mba = \sqrt{\mbA} \iff \mbA = \mba \mba^\top$. Then,
the It\^o diffusion:
\begin{align}
\label{eq:stationary}
    d \mby  = 
    \underbrace{-\Big[\mbQ(s)+\mbD(s)\Big]
    \mbS \mby}_{f(\mby,s)} ds + \underbrace{\sqrt{2 \mbD(s)}}_{g(s)}d\widehat{\mbB}_s,
\end{align}
has Gaussian stationary distribution $\mathcal{N}(\mb0, \mbS^{-1})$ where $\mbQ(s),\mbD(s)$ and $\mbS$ are parameters. For a discussion of convergence to the stationary distribution, 
as well as skew-symmetric and positive semi-definite parameterizations, see \cref{appsec:stationary}, where we also show that existing
diffusion processes such as \gls{vpsde} and \gls{cld} are included
in $\mbQ/\mbD$ parameterization. We display the \gls{elbo} in terms of  $\mbQ/\mbD$
in \cref{appsec:stationary_elbo} and an algorithm in \cref{appsec:algo}.

For score matching and noise prediction losses
and a given $q_\phi$, 
achieving a minimizing value with respect to $s_\theta$ does not imply that
the generative model score will match the inference score. Modeling the data also requires the marginal distribution of $q_{\phi,T}$ to approximate $\pi$. When $q_\phi$ is constant, it is important to confirm the stationary distribution is appropriately set, and the tools used here for the \gls{elbo} can be used to satisfy this requirement for score matching and noise prediction
(\cref{appsec:score_matching}).

\subsection{Learning the inference process}

The choice of diffusion matters, and the \glspl{elbo} in \cref{eq:mdm_bound} have no requirement
for fixed $q_\phi$. We therefore
learn the inference process jointly with $s_\theta$. 
Under linear transitions (ingredient $1$), no algorithmic details
change as the diffusion changes during training. Under stationary
parameterization (ingredient $2$), we can learn without the stationary distribution going awry.
In the experiments, learning matches or surpasses \glspl{bpd} of 
fixed diffusions for a given dataset and score architecture.

In 
$\mathcal{L}^{\text{mdsm}}$
or 
$\mathcal{L}^{\text{mism}}$,
$q_{\phi,\infty}$ may be set to equal $\pi_\theta$, but
it is  $\rvy_T \sim q_{\phi,T}$ for the chosen $T$ that is featured
in the \gls{elbo}. Learning $q_\phi$ can 
choose $\mby_T$ to reduce the cross-entropy:
\begin{align}
    \label{eq:cross_ent}
    -\E_{q_\phi(\mby_T|x)}[\log \pi_\theta(\mby_T)].
\end{align}
Minimizing \cref{eq:cross_ent} will tighten the \gls{elbo} for any $s_\theta$. Next, $q_\phi$ is featured in the remaining terms that feature $s_\theta$; optimizing for $q_\phi$ will tighten and improve the \gls{elbo} alongside $s_\theta$. Finally, $q_\phi$ is featured in the expectations
and the $- \log q_\phi$ term:
\begin{align}
 \log p_\theta(\mbu_T^z=x)
  & \geq 
  = 
    \underbrace{\E_{q_\phi(\mby_0^v=v|x)}}
    \Big[ 
    (\mathcal{L}^{\text{dsm}} \text{ or } \mathcal{L}^{\text{ism}})
     \underbrace{- \log q_\phi(\mby_0^v=v|x)}
     \Big]
\end{align}

The $q_\phi(\mby_0^v|x)$ terms impose an optimality
condition that $p_\theta(\mbu_T^v|\mbu_T^z)=q_\phi(\mby_0^v|\mby_0^z)$ (\cref{appsec:change_of_measure}),
When it is satisfied, no looseness in the \gls{elbo} is due to the initial time zero auxiliary variables. 

To learn, $\mbQ, \mbD$ need to be specified with parameters $\phi$ that enable gradients. We keep $\mbS$ fixed at inverse covariance of $\pi_\theta$.
The transition kernel $q_\phi(\mby_s|x)$
depends on $\mbQ,\mbD$ through its mean and covariance.
Gaussian distributions permit gradient
estimation with  reparameterization or score-function gradients \citep{kingma2013auto,ranganath2014black,rezende2015variational,titsias2014doubly}. Reparameterization is accomplished via:
\begin{align}
\label{eq:reparm}
    \mby_s = \rvm_{s|0} + \rmL_{s|0}\epsilon
\end{align}
where $\epsilon \sim \gN(0, I_{dK})$ and $\rmL_{s|0}$ satisfies $\rmL_{s|0} \rmL_{s|0}^\top = \mathbf{\Sigma}_{s|0}$, derived using coordinate-wise Cholesky decomposition.
Gradients flow through \cref{eq:reparm}
from $\mby_s$ to $\mbm_{s|0}$ and $\boldsymbol{\Sigma}_{s|0}$ to $\mbQ,\mbD$ to parameters $\phi$. 

\Cref{alg:amdt} displays \glsreset{amdt}\gls{amdt}. \Gls{amdt} provides a training method for diffusion-based generative models for either fixed $\rmQ, \rmD$ matrices or for learning the $\rmQ_\phi, \rmD_\phi$ matrices, without requiring any diffusion-specific analysis.

\paragraph{Learning in other diffusion objectives.} Like in the \gls{elbo}, learning in score matching
or noise prediction objectives can improve the match between the inference process and implied generative model (\cref{appsec:score_matching}).

\section{Insights into  Multivariate Diffusions \label{sec:insights}} 
 
\paragraph{Scalar versus Multivariate Processes.}
\Cref{eq:stationary} clarifies what can change while preserving $q_{\infty}$. Recall that $\mbQ$ and $\mbD$ are $K \times K$ for $K-1$ auxiliary variables.
Because $0$ is the only $1 \times 1$ skew-symmetric matrix,
scalar processes must set $\mbQ=0$. 
With $q_{\phi,\infty}=\mathcal{N}(0,\mbI)$, the process is:
\begin{align}
    \label{eq:vp}
     d\mby = -\mbD(s)\mby ds + \sqrt{2\mbD(s)}d\widehat{\mbB}_s.
\end{align}
What is left is 
the \gls{vpsde} process used widely in diffusion models where 
$\mbD(s)=\frac{1}{2}\beta(s)$ 
is $1 \times 1$ \citep{song2020score}.
This reveals that the \gls{vpsde} process is the only scalar diffusion with a stationary distribution.\footnote{There are processes such as sub-\gls{vpsde} \citep{song2020score} which are covered in the sense that they tend
to members of this parameterization as $T$ grows: sub-VP converges to \gls{vpsde}.}
This also clarifies the role of $\mbQ$: it accounts for mixing between dimensions in multivariate processes, as do non-diagonal entries in $\mbD$ for $K > 1$.
\paragraph{\gls{cld} optimizes a
log-likelihood lower bound.}
Differentiating
$\mathcal{L}^\text{mdsm}$ (\cref{eq:mdm_bound})
with respect to the score model parameters,
we show that the objective for \gls{cld} \citep{dockhorn2021score}
maximizes a lower bound on $\log p_\theta(x)$, not just $\log p_\theta(\mbu_0)$, without appealing to the probability flow \textsc{ode}.

\paragraph{Does my model use auxiliary variables?}
An example initial distribution
is
$q(\mby_0^v|x)=\mathcal{N}(0,\mbI)$.
It is also common to set $\pi_\theta=\mathcal{N}(0,\mbI)$.
Because the optimum for diffusions
is $p_\theta=q$,
the optimal model has main and auxiliary dimensions independent at endpoints $0$ and $T$. Does this mean that the model does not use auxiliary variables?
In \cref{appsec:use_aux}, we show that in this case
the model can still use auxiliary variables at intermediate times. A sufficient condition is non-diagonal $\mbQ+\mbD$.

\section{Experiments}
We test the \gls{mdm} framework 
with handcrafted and learned diffusions. The handcrafted diffusions are (a) \Gls{alda}, used in \citep{mou2019high} 
 for accelerated gradient-based \gls{mcmc} sampling (\cref{eq:alda}) and (b) \Gls{malda}: a modified version of \gls{alda} (\cref{eq:malda}).
Both have two auxiliary variables. We also learn diffusions with $1$ and $2$ auxiliary variables. We compare with \gls{vpsde} and \gls{elbo}-trained \gls{cld}.

\begin{table}[th]
    \centering
    \caption{\Gls{bpd} upper-bounds on
    image generation for a fixed architecture. \textsc{cifar}-10: learning outperforms \gls{cld}, and both outperform the standard choice of \gls{vpsde}.  
     \textsc{mnist}: learning matches \gls{vpsde}
     while the fixed auxiliary diffusions are worse. 
     \textsc{imagenet}32: all perform similarly. 
     Learning matches or surpasses the best fixed diffusion, while bypassing the need to choose a diffusion.}
    \label{tab:inference_bpds}
    \begin{tabular}{lllll}
        Model & $K$ & \textsc{cifar}-10 & \textsc{imagenet}32   & \textsc{mnist}  \\ \hline
    \gls{vpsde} & 1  &  $3.20$ &  $3.70$  &    $1.26$  \\
    Learned  & 2 &  $3.07$&  $3.71$   &  $1.28$ \\
    Learned  & 3 &  $3.08$&  $3.72$   &  $1.33$\\
     \gls{cld} & 2 &  $3.11$ &  $3.70$  &   $1.35$ \\
    \gls{malda}  & 3   &  $3.13$ &  $3.72$  &  $1.65$ \\ 
     \gls{alda}  & 3  &  $29.43$  &
     $33.08$&  $124.60$ \\
     \hline
    \end{tabular}
\end{table}

\begin{table}[bh]
  \centering
      \caption{Parameter Efficiency.
      The first two rows display diffusions from previous work: \gls{vpsde} and \gls{cld}, both using score models with \textbf{108 million} parameters on \textsc{cifar}-10. We train
      the rest using
   a score model with \textbf{35.7 million} parameters.
  The learned diffusion matches the performance of \gls{vpsde}-large; changes in the inference can account for as much improvement as a 3x increase in score parameters.
  \glspl{bpd} are upper-bounds.
  } 
  \label{tab:inference_params}
  \begin{tabular}{llll}
      Model & $K$ & Parameters & \textsc{cifar}-10  \\ 
      \hline
  \gls{vpsde}-large \citep{song2021maximum} & 1  & 108M & $3.08$  \\ 
    \gls{cld}-large \citep{dockhorn2021score} & 2 & 108M &  $3.31$ \\
  \hline
  Learned  & 2 & 35.7M & $3.07$ \\
     \gls{cld} & 2 & 35.7M &  $3.11$  \\
  \gls{vpsde} & 1 & 35.7M  &  $3.20$  \\
   \hline
  \end{tabular}

\end{table}

Following prior work,
we train \glspl{dbgm}
for image generation.
We use the U-Net from \citet{ho2020denoising}. We input the auxiliary variables as extra channels, which only increases
the score model parameters in the input and output convolutions (\gls{cld} and Learned 2 have $7,000$ more parameters than \gls{vpsde} on \textsc{cifar}-10 and \textsc{imagenet}32 and only $865$ more for \textsc{mnist}). We use simple uniform dequantization.
 We report estimates of $\mathcal{L}^{\text{mdsm}}$
 (which reduces to the standard $\mathcal{L}^{\text{dsm}}$ for $K=1$).  We sample times using the importance sampling distribution from \cite{song2021maximum} with truncation set to $\epsilon=10^{-3}$. To ensure the truncated bound is proper, we use a likelihood described in \cref{appsec:offset_math}. 

\paragraph{Results.} \Cref{tab:inference_bpds} shows that the inference process matters and displays. It displays \glspl{dbgm} that we train and evaluate on \textsc{cifar}-10, \textsc{imagenet}32 and \textsc{mnist}. This includes the existing \gls{vpsde} and \gls{cld}, the new \gls{malda} and \gls{alda}, and the new learned inference processes. All are trained
with the 35.7M parameter architecture. For \textsc{cifar}-10, learning outperforms \gls{cld}, and both outperform the standard choice of \gls{vpsde}. For \textsc{mnist}, learned diffusions match \gls{vpsde} while the three fixed auxiliary diffusions are worse. On \textsc{imagenet}32, all perform similarly.
    The take-away is that learning matches or surpasses the best fixed diffusion performance and bypasses the choice of diffusion for each new dataset or score architecture.  
In \Cref{fig:samples} we plot the generated samples from \textsc{cifar10}.

\Cref{tab:inference_params}'s first two rows
display diffusion models from previous work: \gls{vpsde} \citep{song2021maximum} and \gls{cld}
\citep{dockhorn2021score} both with
the \textbf{108 million} score model from \cite{song2021maximum} (labeled ``large"). The rest are \glspl{dbgm} that we train using the U-Net with \textbf{35.7 million} parameters for \textsc{cifar}-10 and \textsc{imagenet}32 and 1.1 million for \textsc{mnist}. Despite using  significantly fewer parameters,
the learned diffusion achieves similar \gls{bpd} 
compared to the larger models, showing that changes in inference can account for as much improvement as a three-fold increase in parameters. While the larger architecture requires two GPUs for batch size 128 on \textsc{cifar-10} on A100s, the smaller one only requires one;
exploring inference processes can make
diffusions more computationally accessible.
\Cref{tab:inference_params} also demonstrates a tighter bound for \gls{cld} trained and evaluated with the \gls{mdm} \gls{elbo} ($\leq 3.11$) relative to existing probability flow-based evaluations ($3.31$). 

\vspace{-0.4cm}
\begin{figure}[ht]
    \small
    \centering
    \includegraphics[scale=0.25]{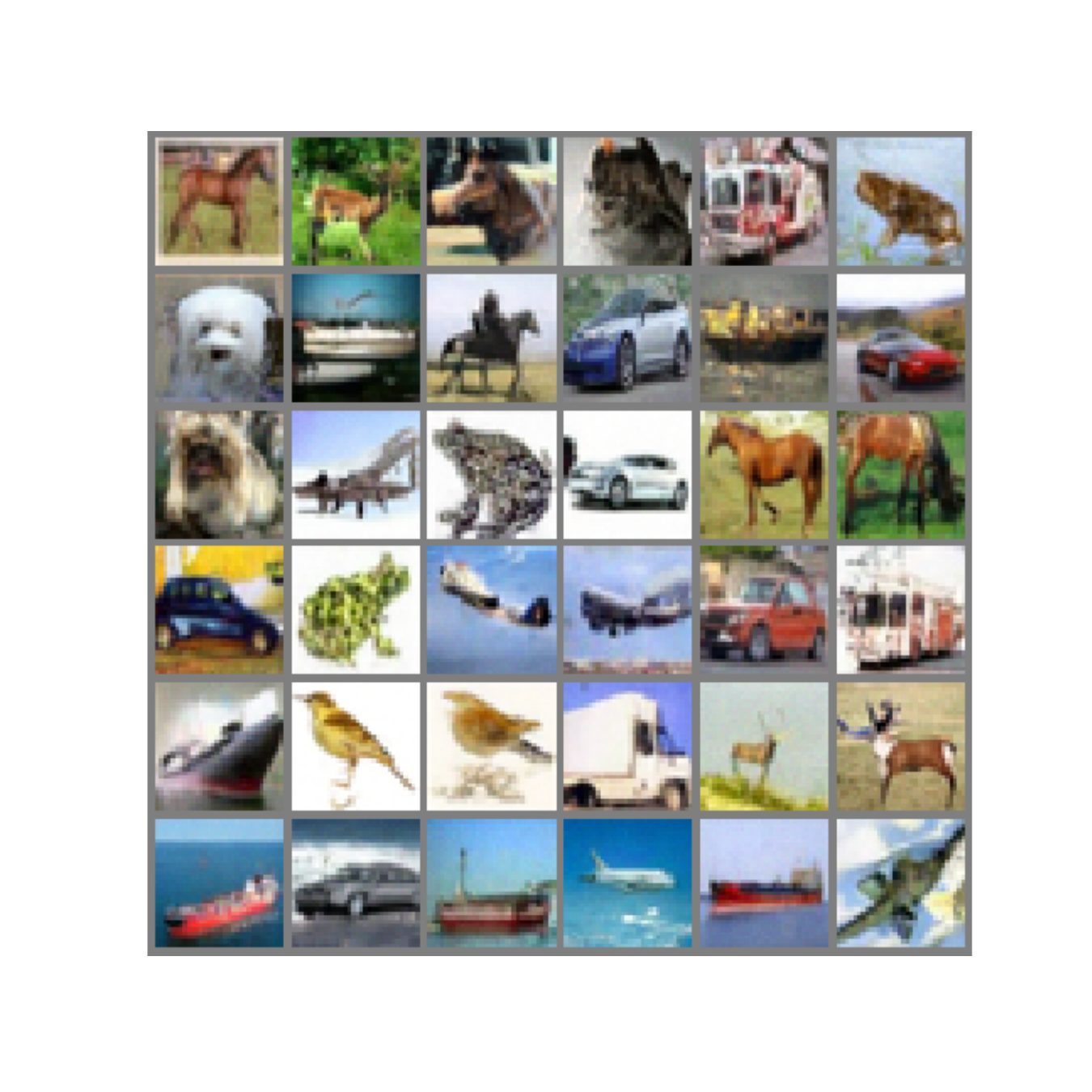}\includegraphics[scale=0.25]{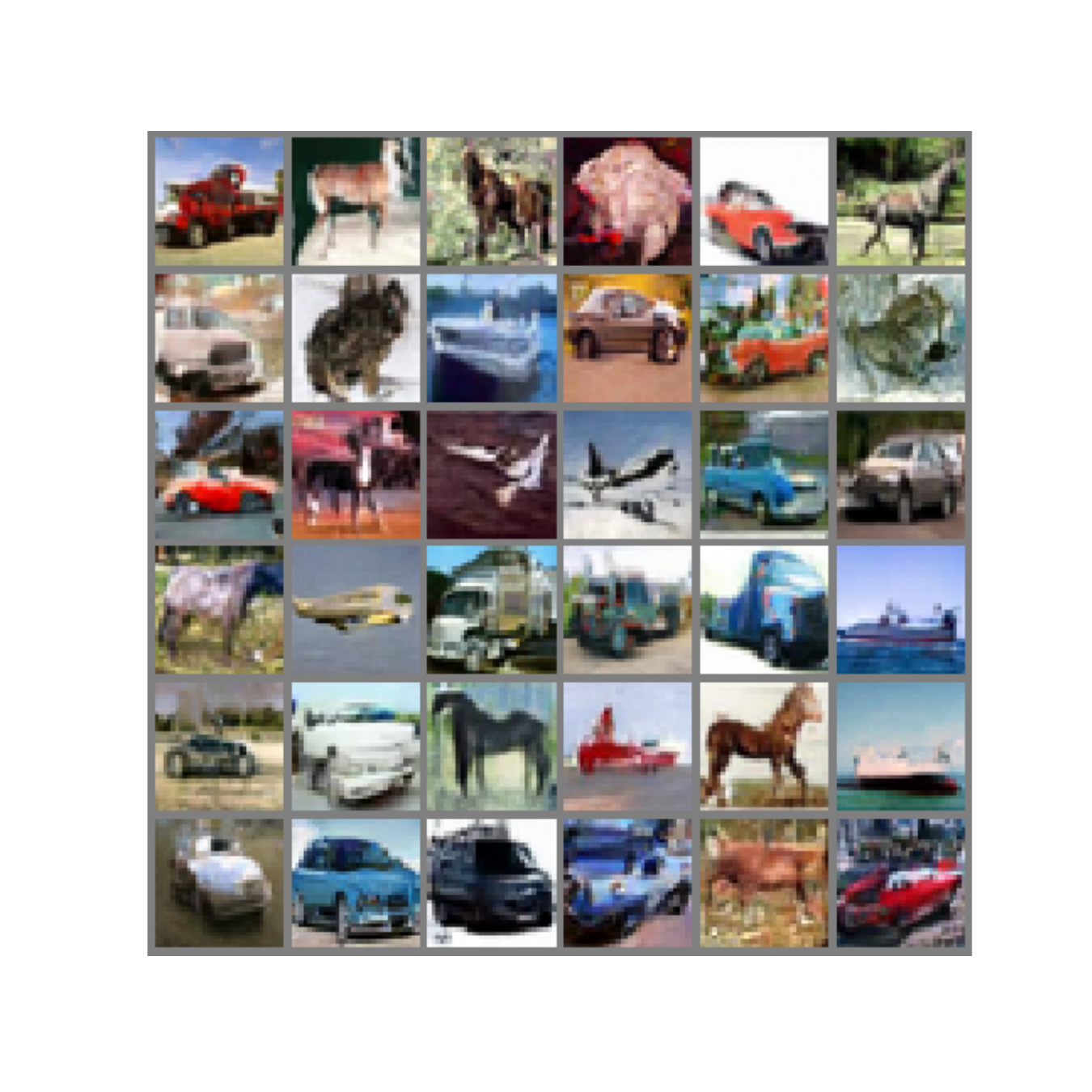}
    \vspace{-0.4cm}    
    \caption{\textsc{cifar}10 samples generated from 
    the ``learned 2" and \gls{malda} generative models.}
    \label{fig:samples}
\end{figure}

\vspace{-0.3cm}
\section{Related Work \label{sec:related}} 

\textbf{Evidence Lower Bounds.} \citet{song2021maximum,huang2021variational} derive the \gls{ism} and \gls{dsm} lower-bounds on the model log likelihood. Our work extends their analysis to the multivariate diffusion setting to derive lower bounds on the log marginal of the data in the presence of auxiliary variables.

\paragraph{Auxiliary variables.}
\citet{dupont2019augmented} shows that
augmented neural \glspl{ode} model
a richer set of functions and
\cite{huang2020augmented} uses this principle 
for normalizing flows. Hierarchical variational models and auto-encoders marginalize auxiliary variables to build expressive distributions \citep{ranganath2016hierarchical,sonderby2016ladder,maaloe2019biva,vahdat2020nvae,child2020very}. 
We apply this principle to \glspl{dbgm}, including
and extending \gls{cld}
 \citep{dockhorn2021score}.

\paragraph{Learning inference.}
Learning $q_\phi$ with $p_\theta$ is motivated in previous work \citep{kingma2013auto,sohl2015deep,kingma2021variational}.
\citet{kingma2021variational} learn the noise schedule for \gls{vpsde}. For \glspl{mdm}, there are parameters to learn beyond the noise schedule; 
$\mbQ$ can be non-zero, $\mbD$ can diagonal or full, give $\mbQ$ and $\mbD$ different time-varying functions, and learn $\nabla \mbH$.

\section{Discussion}
We present an algorithm 
for training multivariate diffusions with linear
time-varying inference processes with a specified stationary distribution
and any number of auxiliary variables.
This includes automating transition kernel computation
and providing a parameterization of diffusions that have a specified stationary distribution, which 
facilitate working with new diffusion processes, including learning the diffusion.
The experiments show that learning matches or surpasses the best fixed diffusion performance, bypassing the need to choose a diffusion. \glspl{mdm} achieve \glspl{bpd} similar to univariate diffusions, with as many as three times more score parameters. The proposed \gls{mdm} \gls{elbo} 
reports a tighter bound for the existing \gls{cld} relative to existing probability flow-based evaluations. This work enables future directions including interactions across data coordinates and using new stationary distributions.
\section{Acknowledgements}

This work was generously funded by NIH/NHLBI Award
R01HL148248, NSF Award 1922658 NRT-HDR: FUTURE
Foundations, Translation, and Responsibility for Data Science, and NSF CAREER Award 2145542.
The authors would additionally like to thank
Chin-Wei Huang for helpful discussing regarding \cite{huang2021variational}.

\bibliography{iclr2022_conference}
\bibliographystyle{iclr2022_conference}

\newpage 

\appendix 

\section{Automated Score Matching with Learned Inference \label{appsec:score_matching}}

Like for the \gls{mdm} \gls{elbo}, the methods in this work 
apply to training with the
score matching loss:
\begin{align*}
    \mathcal{L}_{\textsc{sm}}(x, \theta, \phi) = T \E_{t \sim U[0, T]} \E_{q_\phi(\rvy \mid x)} \left[ \lambda(t) 
    \norm{s_{\theta}(\rvy_t, t) - \nabla_{\rvy_t} \log q_\phi(\rvy_t \mid x)}_2^2 \right] ,
\end{align*}
where $\lambda: [0, T] \rightarrow \sR_{+}$ is a weighing function. The score-matching loss is often optimized in its simplified noise prediction form:
\begin{align*}
     \mathcal{L}_{\textsc{np}}(x, \theta, \phi) = T \E_{t \sim U[0, T]} \E_{q_\phi(\rvy \mid x)} \left[
    \norm{\eps_{\theta}(\rvy_t, t) - \eps}_2^2 \right]
\end{align*}
where $s_\theta = -\rmL_t^{-\top} \eps_\theta$ and $\rvy_t = \mu_t + \rmL_t \eps$ and $\epsilon$ is the noise used in sampling $\mby_t$. 
We describe here
how the improvements to the \gls{elbo}
studied in this work
carry over to 
$\mathcal{L}_{\textsc{SM}}$ and 
$\mathcal{L}_{\textsc{NP}}$. In the following
let $q_0$ be the data distribution, 
let $p_{(\theta,\phi),0}$ be the model's distribution of the data, 
and recall that the model is defined by 
$(s_\theta,f_\phi,g_\phi)$ and prior $\pi$ via a continuous-time stochastic process with drift coefficient $g_\phi^2 s_\theta - f_\phi$ and and diffusion coefficient $g_\phi$.

First, minimizing 
$\mathcal{L}_{\textsc{SM}}$ or
$\mathcal{L}_{\textsc{NP}}$
so that  $\nabla_{\mby_t} \log q_\phi(\mby_t) = s_\theta(\mby_t,t)$
does not alone imply that
$p_{(\theta,\phi),0}$ will equal $q_0$;
it must also be that $q_{\phi,T} \approx \pi$. Foregoing this requirement
means $\pi$ will produce samples that the generative model may not be able to push onto the path 
the model was trained on (formally, the score of the generative model would not equal the time-reversal of the forward score even if 
$s_\theta$ equals the forward score). This condition
can be satisfied if $q_\phi$ can be chosen with stationary distribution $\pi$. Section 3.4 describes how to accomplish this.

Next, for any fixed $q_\phi$, 
automatic transitions from section 3.3 streamline the computation of the score matching loss, allowing for simple score computation for a wide class of diffusions beyond VP.

Finally, for a fixed $q_\phi$ with 
$q_{\phi,T} \approx \pi$
and a score architecture $s_\theta$, 
minimizing 
$\mathcal{L}_{\textsc{SM}}$ or
$\mathcal{L}_{\textsc{NP}}$
w.r.t $\theta$ may be suboptimal. 
Optimization, like for the elbo, carries over to score matching and can 
close this gap; learning w.r.t. both $\theta,\phi$
increases the ability to successfully minimize the loss at each $t$ (section 3.5).
In other words, since the generative model is defined by $(s_\theta, f_\phi, g_\phi)$, learning $q_\phi$ means the loss trains all three components of the generative model rather than just one. In summary, score matching is automatic and can
learn over the space of linear diffusions that tend to the model prior.

\section{Does my model use auxiliary variables? \label{appsec:use_aux}}

In \cref{sec:mdm} we gave the example choice of 
$q(\mby_0^v|x)=\mathcal{N}(0,\mbI)$ coordinate-wise.
It is also a common choice to set $\pi_\theta=\mathcal{N}(0,\mbI)$.
Because the optimum in diffusion models
is $p_\theta=q$ for all $t$, we see
a peculiar phenomenon under this choice:
the model has main and auxiliary dimensions independent at both endpoints $0$ and $T$. Does this mean that the model does not use auxiliary variables?
We show that even when
$q_\phi(\mby_0)$ and $\pi_\theta$ have main and auxiliary variables independent, 
the model can use the auxiliary variables. A sufficient condition is $\mbQ+\mbD$ is non-diagonal.

To make this precise, we recall that we model with $p_\theta(\mbu_T^z=x)$. To show the model is using auxiliary variables, we just need to show that
    $\mbu_T^z$ (main coordinate at $T$) depends on $\mbu_t^v$ (aux. coordinate at $t$)
    for $T>t$. At optimum,
    $p_\theta(\mbu_T^z,\mbu_t^v)=q_\phi(\mby_0^z,\mby_{T-t}^v)$.
Therefore it is sufficient to show that for some time $s$,
    $q_\phi(\mby_s^v|\mby_0^z) \neq q_\phi(\mby_s^v)$.
    Because  $\mby_0^z$,
    is determined by $x$
    we need to show
    that $q_\phi(\mby_s^v|x)\neq q_\phi(\mby_s^v)$.
    To do that,
    we first derive
    $q(\mby_s|x)$ and then
    marginalize to get
    $q(\mby_s^v|x)$
    from it. Since the former
    is 2D Gaussian,
    the latter is available in terms
    of the former's mean and covariance.
    Suppose
     $\E[\mby_0^v]=0$,
     $\mbQ=[[0,-1],[1,0]]$ and $\mbD=[[1,0],[0,1]]$
     and we have $s=.1$ We have:
    \begin{align}
        \E[\mby_s|x]
        &=
        \exp 
        \Big[ 
        -s(\mbQ+\mbD)
        \Big] 
         \begin{pmatrix}
          x\\
          0
        \end{pmatrix}
    =\exp 
        \Big[ 
        \begin{bmatrix}
        -.1 & .1\\
        -.1 & -.1
        \end{bmatrix}
        \Big] 
         \begin{pmatrix}
          x\\
          0
        \end{pmatrix}
        =
        \begin{pmatrix}
          0.9003x\\
          -0.090x
        \end{pmatrix}
    \end{align}
    Regardless of the covariance
    any 1D of this 2D gaussian
    will have mean that is a function of $x$,
    meaning that $q(\mby_s^v|x)$
    does not equal
    $q(\mby_s^v)$ (which is also a 
    Gaussian but with mean depending
    on $\mbx's$ mean rather than $x$ itself.
    Therefore, 
    even under the setup with independent endpoints,
    the optimal model makes use of the intermediate auxiliary variables in its final modeling distribution $p_\theta(\mbu_T^z=x)$.
    
Are there choices of $\mbQ$ and $\mbD$
that lead to learning models that don't make use of the extra dimensions? 
As mentioned, in the inference process, $\mbQ$ is responsible for mixing information among the coordinates, and is the only source of this when $\mbD$ is diagonal. Then, if $\mbQ=\boldsymbol{0}$ and $\mbD$ is diagonal, none of the coordinates
for a given feature $\mbx_j$ (including
$\mbu_{tj}^z,\mbu_{tj}^{v_1},\ldots,\mbu_{tj}^{v_{K-1}}$) interact for any $t$. Then, since $p_\theta=q$ at optimum, independence of the coodinates at all $t$ in $q$ imply the same in $p_\theta$ and the model will not make use of any auxiliary variables when modeling the marginal $\log p_\theta(\mbu_T^z=x)$.

\section{Stationary parameterization \label{appsec:stationary}}
The non-linear time-homogeneous It\^o process family is:
\begin{align}
    d\mby 
    =
    f(\mby)dt
    +
    g(\mby) \mbB_t.
\end{align}
This family can be restricted to those with stationary distributions.
\cite{ma2015complete} show a complete recipe to span
the subset of this family with a desired stationary distribution.
Let $\mbQ$ be skew-symmetric ($-\mbQ=\mbQ^\top$) and $\mbD$ is positive semi-definite. Suppose the desired stationary distribution 
is $q_\infty(\mby)$. 
For a matrix $\mbA$, let $\sqrt{\mbA}$ refer to the matrix square root defined by $\mba = \sqrt{\mbA} \iff \mbA = \mba \mba^\top$. Then,
\cite{ma2015complete} show that,
setting
$\mbH(\mby)=-\log q_{\infty}(\mby)$, $g(\mby)=\sqrt{2\mbD(\mby)}$,
and
\begin{align}
f(\mby)=-[\mbD(\mby)+\mbQ(\mby)] \nabla \mbH(\mby) + \boldsymbol{\Gamma}(\mby), \quad \quad  \boldsymbol{\Gamma}_i(\mby) = \sum_{j=1}^d  \frac{\partial}{\partial \mbz_j} (\mbD_{ij}(\mby) + \mbQ_{ij}(\mby)),
\end{align}
yields a process $\mby_t$ with stationary distribution $q_{\infty}$. We extend it to time-varying (time in-homogeneous) processes.\\

\begin{theorem}
$q_\infty(\mby) \propto \exp[-H(\mby)]$ is a stationary distribution
of
\begin{align}
    d\mby 
    =
    \Bigg(
    -[\mbD(\mby,t)+\mbQ(\mby,t)] \nabla \mbH(\mby) + \boldsymbol{\Gamma}(\mby,t)
    \Bigg)dt
    +
   \sqrt{2\mbD(\mby,t)} \mbB_t,
\end{align}
for 
\begin{align}
\boldsymbol{\Gamma}_i(\mby,t) = \sum_{j=1}^d  \frac{\partial}{\partial \mby_j} (\mbD_{ij}(\mby,t) + \mbQ_{ij}(\mby,t)).
\end{align}
\end{theorem} 
\begin{proof} 
The Fokker Planck equation is:
\begin{align}
    \partial_t q(\mby, t)= - \sum_i \frac{\partial}{\partial \mby_i}\Big[f_i(\mby,t)q(\mby,t)\Big] + \sum_{i,j} \frac{\partial^2}{\partial \mby_i \partial \mby_j}\Big[\mbD_{ij}(\mby,t)q(\mby,t)\Big]
\end{align}
A stationary distribution is one where
the Fokker-Planck right hand side is equal to $0$.
To show that the stationary characterization also holds of time-inhomogenous processes with $\mbD(\mby,t)$ and $\mbQ(\mby, t)$,
we take two steps, closely following \cite{yin2006existence, shi2012relation, ma2015complete},
but noting that there is no requirement for $\mbQ,\mbD$ to be free of $t$. First, we show that the Fokker-Plack equation can be re-written as: 
\begin{align}
\partial_t q(\mby,t) = \nabla \cdot 
\Bigg( 
\Big[ \mbD(\mby,t) + \mbQ(\mby,t) \Big] 
\Big[ q(\mby,t) \nabla H(\mby) + \nabla q(\mby, t) \Big] 
\Bigg)
\end{align}
Second, because the whole expression is set to $0$ when 
the inside expression equals $0$
\begin{align}
   q(\mby,t) \nabla H(\mby) + \nabla q(\mby, t) = 0,
\end{align}
we just need to show that this holds when 
$q(\mby,t) =  \exp[-H(\mby)] / \rmZ$. The second step is concluded because 
\begin{align*}
    \Big[ q(\mby,t) \nabla H(\mby) + \nabla q(\mby, t)
    \Big] 
     &=
    \frac{1}{\rmZ} \Big[ 
    \exp[-H(\mby)] \nabla H(\mby) + \nabla \exp[-H(\mby)] 
    \Big] =0 ,
\end{align*}
where $\rmZ$ is the normalization constant of $\exp(-H(y))$.

It only remains to show that Fokker-Plack can be re-written in divergence form with time-dependent $\mbQ,\mbD$.
In the following let $Q_{ijt}$ denote $\mbQ_{ij}(\mby, t)$ and likewise for $D_{ijt}$.
Let $\partial_i$ denote $\frac{\partial}{\partial \mby_i}$ and let it denote $\frac{d}{d \mby_i}$ for scalar functions. We will use 
$[Ax]_i=\sum_j A_{ij}x_j$.
\begin{align*}
    \partial_t q_t 
         &= \nabla \cdot \Big( [\mbD(\mby,t) + \mbQ(\mby,t)][q \nabla H + \nabla q] \Big)\\
         &= \sum_i \partial_i \Big( \Big[[\mbD(\mby,t) + \mbQ(\mby,t)][q \nabla H + \nabla q]\Big]_i \Big)\\
         &= \sum_i \partial_i \sum_j [D_{ijt} + Q_{ijt}][q \nabla H + \nabla q]_j\\
         &= \sum_i \partial_i \sum_j [D_{ijt} + Q_{ijt}][q \partial_j H + \partial_j q]\\        
         &= \sum_i \partial_i \sum_j [D_{ijt} + Q_{ijt}][q \partial_j H]
            +  \sum_i \partial_i \sum_j [D_{ijt} + Q_{ijt}][\partial_j q]\\ 
             &= \sum_i \partial_i \sum_j [D_{ijt} + Q_{ijt}][q \partial_j H]
            +  \sum_i \partial_i \sum_j D_{ijt}[\partial_j q]
            + \sum_i \partial_i \sum_j Q_{ijt}[\partial_j q]  
\end{align*}
We re-write the 2nd and 3rd term. 
Holding $i$ fixed and noting $q$ is scalar, we get the product rule $\sum_j D_{ijt} (\partial_j q) = \sum_j \partial_j [D_{ijt} q] - q\sum_j \partial_j D_{ijt}$ for each $i$, and likewise for $q$:
\begin{align*}
 &\sum_i \partial_i \sum_j [D_{ijt} + Q_{ijt}][q \partial_j H]
        +  \sum_i \partial_i \sum_j D_{ijt}[\partial_j q]
            + \sum_i \partial_i \sum_j Q_{ijt}[\partial_j q]  \\
&= \sum_i \partial_i \sum_j [D_{ijt} + Q_{ijt}][q \partial_j H] 
        + 
        \sum_i \partial_i 
        \sum_j \partial_j [D_{ijt} q] - q\sum_j \partial_j D_{ijt}\\
    &\quad \quad  
        + 
        \sum_i \partial_i
            \sum_j \partial_j [Q_{ijt} q] - q\sum_j \partial_j Q_{ijt} 
\end{align*}
Because $\mbQ(\mby,t)$ is skew-symmetric, we have that 
$\sum_i \partial_i \sum_j \partial_j [Q_{ijt} q] = 0$, leaving
\begin{align*}
     \partial_t q_t 
&= \sum_i \partial_i 
\Bigg[\sum_j [D_{ijt} + Q_{ijt}][q \partial_j H] 
        \Bigg]
        +
        \sum_i \partial_i 
        \Bigg[
        \sum_j \partial_j [D_{ijt} q] - q\sum_j \partial_j D_{ijt} - q\sum_j \partial_j Q_{ijt}\Bigg] \\
&= \sum_i \partial_i 
\Bigg[\sum_j [D_{ijt} + Q_{ijt}][\partial_j H]  q
\Bigg]
        + 
        \sum_i \partial_i 
        \Bigg[ 
    \sum_j \partial_j [D_{ijt} q] - q \sum_j \partial_j (D_{ijt}+Q_{ijt}) 
        \Bigg]\\
&=
\sum_i \partial_i 
\Bigg[ 
\Big( 
\sum_j [D_{ijt} + Q_{ijt}][\partial_j H]
 - \sum_j \partial_j (D_{ijt}+Q_{ijt}) 
\Big)q
\Bigg]
+
\sum_i \sum_j \frac{\partial^2}{\mby_i \mby_j}
(D_{ijt} q)
\\
\end{align*}
Recalling that $f_i(\mby,t)=\Big(-[D+Q]\nabla H + \Gamma\Big)_i$ and again
that $[Ax]_i=\sum_j A_{ij}x_j$, we have equality with the original Fokker-Planck
\begin{align*}
&=
\sum_i \partial_i 
\Bigg[ 
\Big( 
\sum_j [D_{ijt} + Q_{ijt}][\partial_j H]
 - \sum_j \partial_j (D_{ijt}+Q_{ijt}) 
\Big)q
\Bigg]
+
\sum_{ij}  \frac{\partial^2}{\mby_i \mby_j}
(D_{ijt} q)
\\
&=
-\sum_i \frac{\partial}{\partial \mby_i} \Big[f_i(\mby,t)q(\mby,t) \Big]
+ 
\sum_{ij} \frac{\partial^2}{\mby_i \mby_j}\Big[\mbD_{ij}(\mby,t)q(\mby,t) \Big]\\
&= \partial_t q(\mby,t)
\end{align*}
\end{proof} 
We have shown $\exp[-H(\mby)]/\mbZ$
is a stationary distribution of the time-varying non-linear It\^o process:
\begin{align}
    d\mby 
    =
    \Bigg(
    -[\mbD(\mby,t)+\mbQ(\mby,t)] \nabla H(\mby) + \boldsymbol{\Gamma}(\mby,t)
    \Bigg)dt
    +
   \sqrt{2\mbD(\mby,t)} \mbB_t.
\end{align}
However, for some choices of $\mbQ,\mbD$, $\exp[-H(\mby)]/\mbZ$ is not necessarily the unique stationary
distribution. One problematic case can occur as follows. Suppose that 
row $i$ of $(\mbQ+\mbD)$ is all-zero; in this case, $d\mby_i=0$ which implies that $(\rvy_i)_t = (\rvy_i)_0$ for all $t>0$. Then, the initial 
distribution is also a stationary distribution. To rule out such pathological 
diffusions, we make the assumption that $\mbQ+\mbD$ is full rank.
Then, for uniqueness, recall that stationary distributions are the zeros of
\begin{align*}
\partial_t q(\mby,t) = \nabla \cdot 
\Bigg( 
\Big[ \mbD(\mby,t) + \mbQ(\mby,t) \Big] 
\Big[ q(\mby,t) \nabla H(\mby) + \nabla q(\mby, t) \Big] 
\Bigg)
\end{align*}
where the expression is of the form $\mbA \mbv$ for
$\mbA=\mbD(\mby,t) + \mbQ(\mby,t)$ and 
\begin{align*}
    \mbv=\Big[ q(\mby,t) \nabla H(\mby) + \nabla q(\mby, t) \Big].
\end{align*}
Under the assumption that $\mbQ+\mbD$ is full rank, the expression can only be zero when $\mbv$ is zero.
To show uniqueness under the full rank assumption, one must then show that
\begin{align*}
    \nabla q(\mby, t) = - q(\mby,t) \nabla H(\mby).   
\end{align*}
holds only if $q(\mby,t)=\exp[-H(\mby)]/\mbZ$.
Even if $\exp[-H(\mby)]/\mbZ$ is the unique stationary distribution, convergence to that distribution is a question. See \cite{zhang2013new} for more details.

Learning $\mbQ_\phi,\mbD_\phi$ in the \gls{mdm} \gls{elbo} helps push $\mby_T$ to the model prior $\pi_\theta$ and avoid issues like those discussed.

\subsection{Linear Processes}
Next, we specialize this general family to linear It\^o
processes to maintain tractable transition distributions.
A linear process is one where the drift $f(\mby,t)$ and diffusion
$g(\mby,t)$ are linear functions of $\mby$.
We express the drift function of a non-linear time-varying It\^o process with
stationary distribution proportional to $\exp[-H(\mby)]$ 
as 
\begin{align*}
    -(\mbQ(\mby,t)+\mbD(\mby,t))\nabla H(\mby) + \Gamma(\mby,t).
\end{align*}
Next, linear It\^o processes
have Gaussian stationary distributions \citep{sarkka2019applied} 
so $H(\mby)$ must be quadratic and $\nabla H(\mby)$ is linear, and neither are constant in $\mby$. Because $\nabla H(\mby)$ is linear,
it can be expressed as $\mbS \mby$ for some matrix $\mbS$ where
$\mbS$ is the inverse of the covariance matrix. Because $\nabla H$ is multiplied by $\mbQ,\mbD$, this means that $\mbQ,\mbD$ must be free of $\mby$. Recalling that $\Gamma$ is expressed as
a sum of derivatives w.r.t $\mby$ of $\mbQ + \mbD$, this means that $\Gamma$ must satisfy $\Gamma=0$. 
Next, because of the stationary requirement that $g(t)=\sqrt{2\mbD(\mby,t)}$, we can also
conclude by the restriction on $\mbD$ that the diffusion coefficient function must be independent of the state $\mby$.
Our final form 
for linear time-varying processes with stationary distributions
$\mathcal{N}(0,\mbS^{-1})$ is:
\begin{align}
    d \mby  = 
    \underbrace{-\Big[\mbQ(t)+\mbD(t)\Big]
    \mbS \mby}_{f(\mby,t)} dt + \underbrace{\sqrt{2 \mbD(t)}}_{g(t)}d\mbB_t
\end{align}

\subsection{Parameterizing $\mbQ_\phi$}
Suppose $b_q(s)$ is a positive scalar function defined on the time domain with known integral. Suppose
$\tilde{\mbQ}_\phi$ is any matrix. Then $\tilde{\mbQ}_\phi-\tilde{\mbQ}_\phi^\top$ is skew-symmetric
with $\tilde{\mbQ}_{\phi,ij}= - \tilde{\mbQ}_{\phi,ji}$. We can set $\mbQ_\phi$ to
\begin{align}
    \mbQ_\phi(s) = b_q(s) \cdot \Big[ \tilde{\mbQ}_\phi - \tilde{\mbQ}_\phi^\top \Big]
\end{align}
This is a general parameterization of time-independent skew-symmetric matrices, which have number of degrees of freedom equal to the number of entries in one of the triangles of the matrix, excluding the diagonal.

\subsection{Parameterizing $\mbD_\phi$}
Suppose $b_d(s)$ is a positive scalar function defined on the time domain with known integral.
Suppose $\tilde{\mbD}_\phi$ is any matrix. Then $\tilde{\mbD}_\phi \tilde{\mbD}_\phi^\top$
is positive semi-definite and spans all time-independent positive semi-definite matrices. We can set $\mbD_\phi$ to 
\begin{align}
    \mbD_\phi(s) = b_d(s) \cdot \Big[\tilde{\mbD}_\phi\tilde{\mbD}_\phi^\top \Big]
\end{align}
To show $\tilde{\mbD}\tilde{\mbD}^\top$ spans all positive semi-definite matrices: suppose $\mbM$ is positive semi-definite. Then it is square. Then it can be eigen-decomposed into $\mbM=\mbV \mathbold{\Sigma} \mbV^\top$ The degrees of freedom in
$\mbV \mathbold{\Sigma} \mbV^\top$ are just $\mbR=\mbV \sqrt{\mathbold{\Sigma}}$
since $\mbV \mathbold{\Sigma} \mbV^\top = \mbR \mbR^\top$
and the square root is taken element-wise because $\mathbold{\Sigma}$ is diagonal and is real because
each $\mathbold{\Sigma}_{ij} \geq 0$, which is true because $\mbM$ is positive semi-definite. Take $\mbD = \mbR$.

In our experiments we parameterize $\rmD$ as a diagonal-only matrix.

\subsection{Integrals}

The known integral requirement comes from the integrals required in the transition kernel, and can be relaxed two possible ways:
\begin{itemize}
    \item numerical integration of function with unknown integral. This is expected to have low error given that the function is scalar-in scalar-out.
    \item Directly parameterize the integral and use auto-grad when needing the functions not-integrated.
\end{itemize}
We stick with the known integrals. In conclusion, the underlying parameters are positive scalar functions $b_q(s),b_d(s)$ defined on the time domain and with known integral, and general matrices $\tilde{\mbQ}_\phi, \tilde{\mbD}_\phi$. 

\subsection{Instances}\label{appsec:instances}

\paragraph{\gls{vpsde}.}

\gls{vpsde} has $K=1$. Consequently, $\mbQ,\mbD$ are $K \times K$. The only $1 \times 1$ skew-symmetric matrix is $0$, so $\mbQ=0$. Setting $\mbD(t)=\frac{1}{2} \beta(t)$ recovers \gls{vpsde}:
\begin{align}
    d \mby = -\frac{\beta(t)}{2} \mby dt
        + \sqrt{\beta(t)}d \mbB_t
\end{align}
$\nabla H(\mby)=\mby$ so $\mbH(\mby) = \frac{1}{2}
\| \mby \|_2^2$. The stationary distribution is $\mathcal{N}(0,\mbI)$.

\paragraph{\gls{cld}.}
The \gls{cld} process (eq $5$ in \citet{dockhorn2021score}) is defined as  
\begin{align*}
    \begin{pmatrix}
      d\rvz_t \\
      d\rvv_r 
    \end{pmatrix} = d\rvy_t = 
    \begin{pmatrix}
      0 & \frac{\beta}{M} \\
      -\beta & -\frac{\Gamma \beta}{M} \\
      \end{pmatrix} \rvy_t
      + \begin{pmatrix}
        0 & 0 \\
        0 & \sqrt{2\Gamma \beta}
      \end{pmatrix} d\mbB_t .
\end{align*}
In $\mbQ/\mbD$ parameterization, we have 
\begin{align*}
    H(\rvy) &= \frac{1}{2} \norm{\rvz}_2^2 + \frac{1}{2M} \norm{\rvv}_2^2, \qquad \nabla_{\rvu} H(\rvy) = \begin{pmatrix}
      \rvz \\
      \frac{1}{M} \rvv
    \end{pmatrix} \\
    \mbQ &= \begin{pmatrix}
      0 & -{\beta} \\ 
      {\beta} & 0 
    \end{pmatrix}, \qquad 
    \mbD = \begin{pmatrix}
      0 & 0 \\
      0 & {\Gamma \beta}
    \end{pmatrix}
\end{align*}

The stationary distribution of this process is:
\begin{align}
 q_{\phi,\infty} \propto \exp(-H(\rvy)) = \gN(\rvz; 0, I_d) \gN(\rvv; 0, M I_d)   
\end{align}

\paragraph{\gls{alda}.} \cite{mou2019high} define a third-order diffusion process for the purpose of gradient-based \textsc{mcmc} sampling. The \gls{alda} diffusion process can be specified as
\begin{align}
    \label{eq:alda}
  \mbQ &=
  \begin{pmatrix}
    0 & -\frac{1}{L} I & 0\\
    \frac{1}{L}I & 0 & -\gamma I \\
    0 & \gamma I & 0
  \end{pmatrix}, \quad
  \mbD =
  \begin{pmatrix}
    0 & 0 & 0 \\
    0 & 0 & 0 \\
    0 & 0 & {\frac{\xi}{L}} I
  \end{pmatrix} .
\end{align}
Note that $\mbQ$ is skew-symmetric and $\mbD$ is positive semi-definite, therefore we have that $q_{t}(\rvu) \rightarrow q_{\phi, \infty}$.
In this case,
\begin{align*}
  q_{\phi,\infty}
                      &= \gN(\mbz;0,\mbI_d)\gN(\rvv_1; 0, \frac{1}{L}\mbI_d) \gN(\rvv_2; 0, \frac{1}{L}\mbI_d)
\end{align*}

\paragraph{\gls{malda}.} Similar to \gls{alda}, we specify a diffusion process we term \gls{malda} which 
we specify as
\begin{align}
    \label{eq:malda}
  \mbQ &=
  \begin{pmatrix}
    0 & -\frac{1}{L} I & -\frac{1}{L}\\
    \frac{1}{L}I & 0 & -\gamma I \\
    \frac{1}{L} & \gamma I & 0
  \end{pmatrix}, \quad
  \mbD =
  \begin{pmatrix}
    0 & 0 & 0 \\
    0 & {\frac{1}{L}} I & 0 \\
    0 & 0 & {\frac{1}{L}} I
  \end{pmatrix} .
\end{align}
Note that $\mbQ$ is skew-symmetric and $\mbD$ is positive semi-definite. In this case this is
\begin{align*}
  q_{\phi,\infty}
                      &= \gN(\mbz; 0,\mbI_d)\gN(\mbv_1; 0, \frac{1}{L}I_d) \gN(\mbv_2; 0, \frac{1}{L}I_d)
\end{align*}

\section{Transitions for linear processes \label{appsec:kernel}}

For time variable $s$ and Brownian motion
$\widehat{\mbB}_s$ driving diffusions of the form
\begin{align}
 d \mby = f(\mby,s)ds + g(s)d\widehat{\mbB}_s  , 
\end{align}
when $f_\phi(\mby_s,s),g_\phi(s)$ are linear, the transition kernel $q_\phi(\mby_s|\mby_0)$ is always normal \citep{sarkka2019applied}. Therefore, we just find the mean $\mbm_{s|0}$ and covariance $\mathbold{\Sigma}_{s|0}$ of $q(\mby_s|\mby_0)$. 
Let $f(\mby,s)=\mbA(s)\mby$. The un-conditional time $s$ mean and covariance are solutions to 
\begin{align}
    \label{eq:repeat_odes}
    \begin{split}
    d\mbm_s/ds &= \mbA(s)\mbm_s\\
    d \boldsymbol{\Sigma}_s /ds &= \mbA(s) \boldsymbol{\Sigma}_s + \boldsymbol{\Sigma}_s \mbA^\top(s) + g^2(s) 
    \end{split}
\end{align}
By (6.6) in \citet{sarkka2019applied}, for computing conditionals $q(\mby_s|\mby_0)$, we can take the marginal distribution ODEs and compute conditionals by simply setting the time $0$ mean and covariance initial conditions to the conditioning value and to $\mathbold{0}$ respectively. 
We take (6.36-6.39) and set $\mbm_0=\mbu_0$ and $\mathbold{\Sigma}_0=0$ to condition.
Let $[\mbA]_s=\int_0^s \mbA(\nu)d\nu$.
The mean is 
\begin{align}
\mbm_{s|0}=\exp \left[\int_0^s \mbA(\nu) d\nu \right]
\mby_0 
=
\exp \Big(\Big[ A \Big]_s\Big)
\underbrace{=\exp(s \mbA)\mby_0}_{\text{no integration if $\mbA(\nu)=\mbA$}},
\end{align}
where $\exp$ denotes matrix exponential.  (6.36-6.39) state the covariance $q(\mby_s|\mby_0)$
as a matrix factorization, for which a derivation is provided below
$\mathbold{\Sigma}_{s}=\mbC_s({\mbH_s})^{-1}$ for $\mbC_s,\mbH_s$ being the solutions of:
\begin{align}
    \label{eq:factorization_ode_repeat}
  \begin{pmatrix}
    \frac{d}{ds} \mbC_s \\
    \frac{d}{ds} \mbH_s
  \end{pmatrix} =
  \begin{pmatrix}
    \mbA(s) & g^2(s) \\
    \mathbold{0} & -\mbA^\top(s)
  \end{pmatrix}
        \begin{pmatrix}
          \mbC_s \\
          \mbH_s
        \end{pmatrix}
\end{align}
To condition and get $\mathbold{\Sigma}_{s|0}$ from $\mathbold{\Sigma}_{s}$,
we set $\mathbold{\Sigma}_0=\mathbold{0}$, and initialize $\rmC_s, \rmH_s$ by
$\mbC_0=\mathbold{0}$ and $\mbH_0=\mathbold{I}$.
    \begin{align}
      \begin{pmatrix}
        \mbC_s \\
        \mbH_s
      \end{pmatrix} =
      \exp\Bigg[
      \begin{pmatrix}
        [\mbA]_s & [g^2]_s \\
        \mathbold{0} &  -[\mbA^\top]_s
      \end{pmatrix}
      \Bigg]
            \begin{pmatrix}
              \mathbold{0} \\
              \mathbold{I} 
            \end{pmatrix} 
            \underbrace{=
    \exp\Bigg[s
      \begin{pmatrix}
        \mbA & g^2 \\
        \mathbold{0} & -\mbA^{\top}
      \end{pmatrix}
      \Bigg]
            \begin{pmatrix}
              \mathbold{0} \\
              \mathbold{I}
            \end{pmatrix}}_{\text{no integration if $\mbA(\nu)=\mbA,g(\nu)=g$}}.
    \end{align}
Finally, $\mathbold{\Sigma}_{s|0}=\mbC_s({\mbH_s})^{-1}$.

\subsection{Derivation of the Covariance matrix solution}

\Cref{eq:repeat_odes} gives an expression
for $d\boldsymbol{\Sigma}_s/ds$. To derive the matrix factorization technique used in 
\cref{eq:factorization_ode_repeat},
we use \cref{eq:repeat_odes} 
and the desired condition
$\boldsymbol{\Sigma}_s=\mbC_s \mbH_s^{-1}$
to derive expressions for $d\mbC_s/ds$ 
and $d \mbH_s / ds$ and suitable intial conditions
so that the factorization also starts
at the desired $\boldsymbol{\Sigma}_0$. Let $\mathbold{\Sigma}_{s} = \rmC_s \rmH_s^{-1}$, then note that $\rmC_s, \rmH_s$ satisfies
\begin{align*}
    \frac{d}{ds} \mathbold{\Sigma}_{s} &= \frac{d}{ds} \rmC_s \rmH_s^{-1} \\
    &= \rmC_s \frac{d}{ds} \rmH_s^{-1} +  \left(\frac{d}{ds} \rmC_s \right) \rmH_s^{-1}
\end{align*}
And using the fact that 
\begin{align*}
    \frac{d}{ds} \rmH_s \rmH_s^{-1} &= 0 \\
    \rmH_s \frac{d}{ds} \rmH_s^{-1} + \frac{d}{ds} \rmH_s  \left(\rmH_s^{-1}\right)  &= 0 \\
    \frac{d}{ds} \rmH_s^{-1} &= - \rmH_s^{-1} \frac{d}{ds} \rmH_s  \left(\rmH_s^{-1}\right) 
\end{align*}
we get that 
\begin{align*}
    \rmC_s \frac{d}{ds} \rmH_s^{-1} +  \left(\frac{d}{ds} \rmC_s \right) \rmH_s^{-1} 
    &= - \rmC_s \rmH_s^{-1} \frac{d}{ds} \rmH_s  \left(\rmH_s^{-1}\right) + \left(\frac{d}{ds} \rmC_s \right) \rmH_s^{-1} \\
   - \rmC_s \rmH_s^{-1} \frac{d}{ds} \rmH_s  \left(\rmH_s^{-1}\right) + \left(\frac{d}{ds} \rmC_s \right) \rmH_s^{-1} &= \rmA(s) \rmC_s \rmH_s^{-1} + \rmC_s \rmH_s^{-1} \rmA^\top(s) + g^2(s) \\
    &= \rmA(s) \rmC_s \rmH_s^{-1} + \rmC_s \rmH_s^{-1} \rmA^\top(s) \rmH_s \rmH_s^{-1} + g^2(s)\rmH_s \rmH_s^{-1} \\
   \left( - \rmC_s \rmH_s^{-1} \frac{d}{ds} \rmH_s   + \frac{d}{ds} \rmC_s   \right) \rmH_s^{-1} &= 
   \left( \rmA(s) \rmC_s + \rmC_s \rmH_s^{-1}  \rmA^\top(s) \rmH_s + g^2(s) \rmH_s  \right) \rmH_s^{-1} \\
   - \rmC_s \rmH_s^{-1} \frac{d}{ds} \rmH_s   + \frac{d}{ds} \rmC_s    &= 
   \rmA(s) \rmC_s + \rmC_s \rmH_s^{-1}   \rmA^\top(s) \rmH_s + g^2(s) \rmH_s \\
   \left[ \rmC_s \rmH_s^{-1} \quad \rmI_d \right]^\top 
  \frac{d}{ds} \begin{pmatrix}
     \rmH_s \\
     \rmC_s
   \end{pmatrix}
   &= \left[ \rmC_s \rmH_s^{-1} \quad \rmI_d \right]^\top
   \begin{pmatrix}
     -\rmA^\top(s) \rmH_s \\
     \rmA(s) \rmC_s + g^2(s) \rmH_s
   \end{pmatrix}
\end{align*}
Now, we note $\rmC_s, \rmH_s$ satisfy the following
\begin{align*}
    \frac{d}{ds} \rmH_s &= -\rmA^{\top}(s) \rmH_s \\
     \frac{d}{ds} \rmC_s &= \rmA(s)  \rmC_s + g^2(s) \rmH_s
\end{align*}
which implies that 
\begin{align}
\frac{d}{ds}
  \begin{pmatrix}
     \mbC_s \\
     \mbH_s
  \end{pmatrix} =
  \begin{pmatrix}
    \mbA(s) & g^2(s) \\
    \mathbold{0} & -\mbA^\top(s)
  \end{pmatrix}
        \begin{pmatrix}
          \mbC_s \\
          \mbH_s
        \end{pmatrix}
\end{align}
with $\rmC_0 = \mathbold{\Sigma}_0$ and $\rmH_0 = \rmI_d$, as $\rmC_0 \rmH_0^{-1} = \mathbold{\Sigma}_0$.

\subsection{Hybrid Score Matching}

Instead of computing $q(\mby_s|\mby_0)$, we can apply the hybrid score matching principle \citep{dockhorn2021score}
to reduce variance by compute objectives using $q(\mby_s|x)$ instead of $q(\mby_s|\mby_0)$, which amounts to integrating out $\mbv_0$.  To accomplish this,
following \cite{sarkka2019applied},
we simply replace
$\mby_0$ with $[x,\E[\mbv_0]]$
in the expression for $\mathbold{m}_{s|0}$,
i.e. replace the conditioning value of $\mbv_0$ with the mean of its chosen initial distribution:
\begin{align}
    \E[\mby_s|x]
    =
    \exp
    \Bigg[ 
    \int_0^s A(\nu) d\nu 
    \Bigg] 
    \begin{pmatrix}
      x\\
      \E[\mbv_0]
    \end{pmatrix}
\end{align}

For the convariance,
instead of using $\mathbold{C}_0=\mathbold{\Sigma}_{0}=\mathbold{0}$,
we use a block matrix
to condition on $x$ but not $\mbv_0$.
We decompose $\mathbold{\Sigma}_0$ into its blocks
$\mathbold{\Sigma}_{0,xx}$,
$\mathbold{\Sigma}_{0,vv}$
,$\mathbold{\Sigma}_{0,xv}$.
As before, to condition on $x$ we set $\mathbold{\Sigma}_{0,xx}=\mathbold{0}$.
Because $q(\mbv_0)$ is set to be independent of $x$,
$\mathbold{\Sigma}_{0,xv}$ is also set to $\mathbold{0}$. Finally,
instead of $\mathbold{0}$,
to marginalize out  $\mbv_0$,
$\mathbold{\Sigma}_{0,vv}$ is set to the covariance
of the chosen initial time zero distribution for $\mbv_0$. E.g. if $\mbv_{0,j} \sim N(0,\gamma)$ for
each dimension, then $\mathbold{\Sigma}_{0,vv}=N(0,\gamma I)$.

We operationalize this in a simple piece of code,
which makes the \gls{elbo}  tractable and easy, i.e. skips both analytic derivations and numerical forward integration during training.

\subsection{Transitions in Stationary Parameterization}

In terms of $\mbQ,\mbD$, the transitions $q(\mby_s | \mby_0)$ for time $s$ are normal with mean $\mbm_{s|0}$ and $\mathbold{\Sigma}_{s|0}$ equal to:
\begin{align}
    \label{eq:QD_transition_t}
    \mbm_{s|0}=\exp \Big(-
    \Big[\mbQ + \mbD \Big]_s
    \Big) \mby_0,
    \quad \quad 
        \begin{pmatrix}
        \mbC_s \\
        \mbH_s
      \end{pmatrix} =
      \exp\Bigg[
      \begin{pmatrix}
        -[\mbQ + \mbD]_s & [2\mbD]_s \\
        \mathbold{0} &  [(\mbQ + \mbD)^\top]_s
      \end{pmatrix}
      \Bigg]
            \begin{pmatrix}
              \mathbold{0} \\
              \mathbold{I} 
            \end{pmatrix} 
\end{align}
where $\mathbold{\Sigma}_{s|0}=\mbC_s({\mbH_s})^{-1}$.
For the time invariant case, this simplifies to 
\begin{align}
    \label{eq:QD_transition}
    \mbm_{s|0}=\exp[-s(\mbQ + \mbD)]\mby_0, \quad \quad
      \begin{pmatrix}
        \mbC_s \\
        \mbH_s
      \end{pmatrix} =
    \exp\Bigg[s
      \begin{pmatrix}
        -(\mbQ + \mbD) & 2 \mbD \\
        \mathbold{0} & (\mbQ + \mbD)^{\top}
      \end{pmatrix}
      \Bigg]
            \begin{pmatrix}
              \mathbold{0} \\
              \mathbold{I}
            \end{pmatrix}
    \end{align}
    
\section{Generic change of measure and Jensen's for approximate marginalization \label{appsec:change_of_measure}}

Suppose $\mbu=[\mbz,\mbv]$ and we have an expression for
$p(\mbu=[z,v])=p(\mbz=z,\mbv=v)$. By marginalization, we can get $p(\mbz=z)$, and we can introduce another distribution $q$ to pick a sampling distribution of our choice:
\begin{align}
    \label{eq:marginal_expect}
\begin{split}
        p(\mbz=z) &= \int_v p(\mbz=z,\mbv=v) dv\\
                       &= \int_v 
                        p(\mbz=z|\mbv=v)p(\mbv=v) dv\\
                       &= \int_v \frac{q(\mbv=v|\mbz=z)}{q(\mbv=v|\mbz=z)}p(\mbz=z|\mbv=v)p(\mbv=v) dv\\
                       &=
                       \E_{q(\mbv=v|\mbz=z)}
                        \Big[ 
                            \frac{p(\mbz=z,\mbv=v)}{q(\mbv=v|\mbz=z)}
                        \Big] 
\end{split}
\end{align}
We often work with these expressions in log space, and need to pull the expectation outside to use Monte Carlo. Jensen's bound allows this:
\begin{align*}
    \log p(\mbz=z) &= \log \E_{q(\mbv=v|\mbz=z)}
                        \Big[ 
                            \frac{p(\mbz=z,\mbv=v)}{q(\mbv=v|\mbz=z)}
                        \Big] \\
                & \geq \E_{q(\mbv=v|\mbz=z)}
                        \Big[ 
                            \log 
                            \frac{p(\mbz=z,\mbv=v)}{q(\mbv=v|\mbz=z)}
                        \Big] 
\end{align*}
The following shows that the bound is tight when
$q(\mbv=v|\mbz=z) = p(\mbv=v|\mbz=z)$:
\begin{align}
    \label{eq:tight_r}
    \begin{split}
    \E_{q(\mbv=v|\mbz=z)}
                        \Big[ 
                            \log 
                            \frac{p(\mbz=z,\mbv=v)}{q(\mbv=v|\mbz=z)}
                        \Big] 
    &=_{\text{assume}}
    \E_{p(\mbv=v|\mbz=z)}
                        \Big[ 
                            \log 
                            \frac{p(\mbz=z,\mbv=v)}{p(\mbv=v|\mbz=z)}
                        \Big] \\
    &=
 \E_{p(\mbv=v |\mbz=z)}
        \Big[ 
                \log\Big( \frac{p(\mbz=z,\mbv=v)}{p(\mbv=v,\mbz=z)} \cdot
                p(\mbz=z) \Big)
        \Big]\\
 &=
 \E_{p(\mbv=v | \mbz=z)}
        \Big[ 
                \log
                p(\mbz=z)
        \Big]\\
&=\log p(\mbz=z)    
    \end{split}
\end{align}

\section{\gls{elbo}  for \glspl{mdm} \label{appsec:mdm_elbo}}

\begin{align}\log p_\theta(x) &= \log \int_{v_0} p_\theta(x_0,v_0)dv_0\\ &= \log \int_{v_0}p_\theta(u_0=[x,v_0])\\
&= \log \int_{v_0}  \frac{q(v_0|x)}{q(v_0|x)}
p_\theta(u_0=[x,v_0])\\
&= \log \mathbb{E}_{q(v_0|x)}
\Bigg[ \frac{p_\theta(u_0=[x,v_0])}{q(v_0|x)}\Bigg]\\
& \geq \mathbb{E}_{q(v_0|x)}
\Bigg[ \log p_\theta(u_0=[x,v_0]) - \log q(v_0|x)\Bigg]\\
& \geq 
\mathbb{E}_{q(y|x)}
\Bigg[ 
\log \pi_\theta(y_T)
+
\int_0^T
- \|s_\theta \|_{g^2}^2 - \nabla \cdot (g^2 s_\theta - f)
ds
 - \log q(y^v_0|x)
\Bigg]
 \end{align}
The first inequality holds due to Jensen's inequality and the second due to an application of Theorem 1 from \cite{huang2021variational} or Theorem 3 from
\cite{song2021maximum} applied to the joint variable $\mbu_0$.

\subsection{ISM to DSM \label{appsec:ism_to_dsm}} 

\subsubsection{Lemma: expectation by parts} 

We will need a form of multivariate integration by parts which gives us for some $f$ and some $q(x)$,  $E_{q(x)} [\nabla_{x} \cdot f(x)] =-E_{q(x)}[f(x)^{\top} \nabla_{x} \log q(x)]$

\begin{align*}
 E_{q(x)} [\nabla_{x} \cdot f_i(x)] 
  &= \int q(x) \sum_{i=1}^d [\nabla_{x_i} f_i(x)] dx \\
  &= \int \sum_{i=1}^{d} q(x) \nabla_{x_i} f_i(x) dx \\
  &= \sum_{i=1}^{d} \int_{x_{-i}} \int_{x_i} q(x) \nabla_{x_i} f_i(x) dx_i dx_{-i} \\
  &= \sum_{i=1}^d \int \Bigg [ \Big[q(x) \int \nabla_{x_i} f_i(x) dx_i \Big]_{-\infty}^{\infty} - \int \nabla_{x_i} q(x) \int \nabla_{x_i} f_i(x) dx_i] \Bigg] dx_{-i} \\
  &= \sum_{i=1}^d \int \Bigg[ - \int \nabla_{x_i} q(x) f_i(x) dx_i \Bigg] dx_{-i} \\
  &= \sum_{i=1}^d \int \Bigg[ - \int q(x) \nabla_{x_i} \log q(x) f_i(x) dx_i \Bigg] dx_{-i} \\
  &= \sum_{i=1}^d - \int \int q(x) \nabla_{x_i} \log q(x) f_i(x) dx_i  dx_{-i} \\
  &= \sum_{i=1}^d - E_{q(x)} \Big[ \nabla_{x_i} \log q(x) f_i(x) \Big] \\
  &= -E_{q(x)}[f(x)^{\top} \nabla_{x} \log q(x)]
\end{align*}
This equality also follows directly from the Stein operator using the generator method to the Langevin diffusion~\citep{barbour1988stein}.

\subsubsection{DSM Elbo}
Using the ``expectation by parts", we have:
\begin{align*}
    \mathbb{E}_{q(u_t|x)}[\nabla_{u_t} \cdot g^2(t) s_\theta(u_t,t)] = -\mathbb{E}_{q(u_t|x)}[(g^2(t)s_\theta(u_t,t))^\top \nabla_{u_t} \log q(u_t|x)]
\end{align*}
Also we have, for $s_\theta$ evaluated at $(u_t,t)$,
by completing the square,
\begin{align*}
    &-\frac{1}{2}||s_\theta||_{g^2(t)} + s_\theta^\top g^2(t) \nabla \log q(u_t|x) = -\frac{1}{2} ||s_\theta - \nabla \log q(u_t|x)||_{g^2(t)}^2 + .5||\nabla \log q(u_t|x)||_{g^2(t)}^2
\end{align*}
The two together give us:
\begin{align}
\begin{split}
    \log p(x)
        &\geq \E_{q(u_T|x)}\Bigg[\log \pi \Bigg] + \int_0^T 
        \Bigg[\mathbb{E}_{q(u_t|x)} \Big[-\nabla \cdot g^2 s_\theta - .5||s_\theta||_{g^2(t)}^2 + \nabla \cdot f \Big] dt \Bigg]\\
        &= \E_{q(u_T|x)}\Bigg[\log \pi \Bigg] + \int_0^T 
        \Bigg[\mathbb{E}_{q(u_t|x)} \Big[(g^2 s_\theta)^\top \nabla_{u_t} \log q(u_t|x) -.5||s_\theta||_{g^2(t)}^2  +\nabla \cdot f \Big] dt \Bigg]\\
                &= \E_{q(u_T|x)}\Bigg[\log \pi \Bigg] +
        \int_0^T 
        \Bigg[\mathbb{E}_{q(u_t|x)} \Big[-\frac{1}{2} ||s_\theta - \nabla \log q(u_t|x)||_{g^2(t)}^2 
        \\
        & \hspace{18em} + .5||\nabla \log q(u_t|x)||_{g^2(t)}^2  + \nabla_{u_t} \cdot f \Big]  \Bigg]dt
\end{split}
\end{align}

\subsection{Noise prediction}
We have that for normal $\mathcal{N}(\mby_s ; \mbm_{s|0},
\boldsymbol{\Sigma}_{s|0})$,
we can sample $\mby_s$ with normal noise $\epsilon \sim \mathcal{N}(0,I)$ and
$\mby_s = \mbm_{s|0} + \mbL \epsilon$ where $\mbL$ is the cholesky decomposition of $\boldsymbol{\Sigma}_{s|0}$ Then, the score is
\begin{align*}
\nabla_{\mby_s}
    &\log q(\mby_s | \mby_0) \Bigg|_{\mby_s = \mbm_{s|0} + \mbL \epsilon} \\
    &=
    -\boldsymbol{\Sigma}_{s|0}^{-1}
    \Big(\mby_s -\mbm_{s|0} \Big)\\
    &=-\boldsymbol{\Sigma}_{s|0}^{-1}
    \Big(\Big[\mbm_{s|0} + \mbL \epsilon\Big] -\mbm_{s|0} \Big)\\
    &=-\boldsymbol{\Sigma}_{s|0}^{-1}
    \Big(\mbL \epsilon\Big)\\
    &=
    -
    \Big(\mbL \mbL^\top\Big)^{-1}
    \Big(\mbL \epsilon\Big)\\
        &=
    -
    \Big(\mbL^\top\Big)^{-1} \mbL^{-1}
    \mbL \epsilon \\
       &=
    -
    \Big(\mbL^\top\Big)^{-1} \epsilon
    =
     -
    \Big(\mbL^{-1}\Big)^\top\epsilon
    =
    -\mbL^{\top,-1} \epsilon
\end{align*}

Parameterize $s_\theta(\mby_s,s)$
as $s_\theta(\mby_s,s)
    =
    -\mbL^{\top,-1} \epsilon_\theta(\mby,s)$.
This gives
\begin{align*}
    &\frac{1}{2} 
    \| 
    -\mbL^{\top,-1} \epsilon_\theta(\mby,s)
    \quad
    -
    \quad 
    -\mbL^{\top,-1} \epsilon
   \|_{g_\phi^2(s)}^2\\
    &=
    \frac{1}{2} 
    \|
    \mbL^{\top,-1} \epsilon \quad 
    -
    \quad 
    \mbL^{\top,-1} \epsilon_\theta(\mby,s)
    \|_{g_\phi^2(s)}^2\\
    &=
    \frac{1}{2}
    \Bigg( 
    \mbL^{\top,-1} \epsilon \quad 
    -
    \quad 
    \mbL^{\top,-1} \epsilon_\theta(\mby,s)
    \Bigg)^\top 
    g_\phi^2(s)
    \Bigg( 
    \mbL^{\top,-1} \epsilon \quad 
    -
    \quad 
    \mbL^{\top,-1} \epsilon_\theta(\mby,s)
    \Bigg)\\
    &=
    \frac{1}{2}
    \Bigg( 
    \mbL^{\top,-1}
    \Big[ 
    \epsilon - \epsilon_\theta(\mby,s)
    \Big]
    \Bigg)^\top 
    g_\phi^2(s)
   \Bigg( 
    \mbL^{\top,-1}
    \Big[ 
    \epsilon - \epsilon_\theta(\mby,s)
    \Big]
    \Bigg)
\end{align*}
We can also use this insight to analytically compute the quadratic score term (following is computed per data-dimension, so must be multiplied by $D$ when computing the \gls{elbo}):
\begin{align*}
\E_{\mby_0}
\E_{\mby_s|\mby_0}
\Bigg[ 
    \frac{1}{2}
    \| 
    \nabla_{\mby_s}
    \log q_\phi(\mby_s|\mby_0)
    \|^2_{g_\phi^2(s)}
\Bigg] 
    &=
    \E_{\mby_0}
\E_{\mby_s|\mby_0}
\Bigg[ 
    \Big(
    \nabla_{\mby_s}
    \log q_\phi(\mby_s|\mby_0)
    \Big)^\top 
    g_\phi^2(s)
    \Big(
    \nabla_{\mby_s}
    \log q_\phi(\mby_s|\mby_0)
    \Big)
    \Bigg]\\
    &=
    \E_{\mby_0}
\E_{\mby_s|\mby_0}
\Bigg[ 
    \Big(
     -\mbL^{\top,-1} \epsilon
    \Big)^\top 
    g_\phi^2(s)
    \Big(
    -\mbL^{\top,-1} \epsilon
    \Big)\Bigg]\\
        &=
        \E_{\mby_0}
\E_{\mby_s|\mby_0}
\Bigg[ 
     \epsilon^\top 
     (-\mbL^{-1})
    g_\phi^2(s)
    (-\mbL^{\top,-1}) \epsilon \Bigg]\\
    &=
    \E_{\mby_0}
\E_{\mby_s|\mby_0}
\Bigg[ 
     \epsilon^\top 
     \Big(
     \mbL^{-1}
    g_\phi^2(s)
    \mbL^{\top,-1}
    \Big)\epsilon
\Bigg]\\
    &=
    \E_{\mby_0}
\E_{\epsilon}
\Bigg[ 
     \epsilon^\top 
     \Big(
     \mbL^{-1}
    g_\phi^2(s)
    \mbL^{\top,-1}
    \Big)\epsilon
\Bigg]\\
  &=
\E_{\epsilon}
\Bigg[ 
     \epsilon^\top 
     \Big(
     \mbL^{-1}
    g_\phi^2(s)
    \mbL^{\top,-1}
    \Big)\epsilon
\Bigg]\\
&=
\text{Trace}
\Bigg( 
\mbL^{-1}
    g_\phi^2(s)
    \mbL^{\top,-1}
\Bigg)
\end{align*}

\section{elbos in stationary parameterization \label{appsec:stationary_elbo}}

We use the stationary parmeterization described in 
\cref{appsec:stationary}.
We now specialize the \gls{elbo}  to the linear stationary parameterization.

Recall $f_\phi(\mby,s)=-[\mbQ_\phi(s)+\mbD_\phi(s)]\mby$.
Recall $g_\phi(s)=\sqrt{2\mbD_\phi(s)}$
We have $g_\phi^2(s) = 2 \mbD_\phi(s)$.
We can write the \gls{mdm} \gls{ism} \gls{elbo} as
\begin{align}
\mathcal{L}^{\text{mism}}=
    \E_{v \sim q_\gamma}
    \Bigg[
        \E_{s \sim \text{Unif}(0,T)}
        \Big[ 
            \ell_{s}^{(ism)}
        \Big]
        +
        \ell_{T}
        + \ell_q
    \Bigg]
\end{align}
where
\begin{align}
\begin{split}
    \ell_{s_\theta}
    &=
    -\frac{1}{2} \|s_\theta(\mby_s,s) \|^2_{\underbrace{2\mbD_\phi(s)}_{g^2_\phi}} 
    \\
    \ell_{\text{div-fgs}}
    &= 
    \nabla_{\mby_s} \cdot 
    \Big[ 
    \underbrace{-[\mbQ_\phi(s) + \mbD_\phi(s)]\mby_s}_{f_\phi}
    -
   \underbrace{2\mbD_\phi(s)}_{g_\phi^2} s_\theta(\mby_s,s) 
    \Big]\\
     \ell_s^{\text{ism}}
        &=
        \E_{\underbrace{q_{\phi,s,(x,v)}}_{\text{depends on } \mbQ,\mbD}}
        \Big[
        \ell_{s_\theta}
        +
        \ell_{\text{div-fgs}}\Big]\\
    \ell_T
    &= \E_{\underbrace{q_{\phi,T},(x,v)}_{{\text{depends on } \mbQ,\mbD}}}
        \Big[
            \log \pi_\theta(\mby_T)
        \Big]\\
    \ell_q
    &= -\log q_\gamma(v|x)
\end{split}
\end{align}

For the \gls{dsm} form, 
\begin{align}
   \mathcal{L}^{\text{mdsm}}=
    \E_{v \sim q_\gamma}
    \Bigg[
        \E_{s \sim \text{Unif}(0,T)}
        \Big[ 
            \ell_{s}^{(dsm)}
        \Big]
        +
        \ell_{T}
        + \ell_q
    \Bigg]
\end{align}
where
\begin{align*}
\ell_{\text{div-f}}
&= 
\nabla_{\mby_s} \cdot 
    \underbrace{-[\mbQ_\phi(s) + \mbD_\phi(s)]\mby_s}_{f_\phi}\\
\ell_{\text{fwd-score}}
&=\frac{1}{2}\Big | \Big |
\underbrace{\nabla_{\mby_s} \log q_\phi(\mby_s|\mby_0)}_{\text{depends on } \mbQ,\mbD}
\|^2_{\underbrace{2\mbD_\phi(s)}_{g^2_\phi}} \\
\ell_{\text{neg-scorediff}}
&=
-\frac{1}{2} \|s_\theta(\mby_s,s)
        - 
        \underbrace{\nabla_{\mby_s} \log q_\phi(\mby_s|\mby_0)}_{\text{depends on } \mbQ,\mbD}
   \|^2_{\underbrace{2\mbD_\phi(s)}_{g^2_\phi}}  \\
    \ell_s^{(dsm)} 
     &=
     \E_{\underbrace{q_{\phi,s,(x,v)}}_{\text{depends on } \mbQ,\mbD}}
\Bigg[
\ell_{\text{neg-scorediff}}
+
\ell_{\text{fwd-score}}
+ \ell_{\text{div-f}}
    \Bigg]
\end{align*}

\newpage

\section{Algorithms \label{appsec:algo}}

\subsection{Generic Transition Kernel}

\begin{algorithm}[h!]
\begin{algorithmic}

 \STATE {\bfseries Input:}
 data $x$. time $s$.
 $\mbA,g$.

 \STATE {\bfseries compute:}
    $\mbA(s)$ and $g(s)$

\STATE {\bfseries compute:}
    $\mbM_s = \int_0^s \mbA(t) dt$ (integrated drift)

\STATE {\bfseries compute:}
    $\mbN_s = \int_0^s g^2(t) dt$ (integrated diffusions squared)

\STATE {\bfseries compute:} 
$\gamma_{s|0}=\exp
    \Big( 
    \mbM_s
    \Big)$ (mean coefficient)
    
    \STATE {\bfseries set:}
      $\mby_0=[x,0_{1},\ldots, 0_{K-1}]$ ,
      $\boldsymbol{\Sigma}_{0,zz}=\boldsymbol{0}$,
      and
    $\boldsymbol{\Sigma}_{0,zv},\boldsymbol{\Sigma}_{0,vv}$ to chosen initial distribution
\STATE {\bfseries compute:}
 $\mbm_{s|0}=\gamma_{s|0}\mby_0$ (\textbf{mean})
\STATE {\bfseries compute:} 
        \begin{align}
            \begin{pmatrix}
        \mbC_s \\
        \mbH_s
      \end{pmatrix} =
      \exp\Bigg[
      \begin{pmatrix}
        \mbM_s & \mbN_s \\
        \mathbold{0} &  -\mbM_s^\top 
      \end{pmatrix}
      \Bigg]
            \begin{pmatrix}
              \mathbold{\Sigma}_0 \\
              \mathbold{I} 
            \end{pmatrix} 
            \quad 
            \text{(ingredients for cov.)}
        \end{align}
\STATE {\bfseries compute:} $\mathbold{\Sigma}_{s|0}=\mbC_s({\mbH_s})^{-1} \quad 
(\textbf{cov.})$
\STATE{ \bfseries Output:} $\mathcal{N}(\mbm_{s|0},\boldsymbol{\Sigma}_{s|0})$

\end{algorithmic}
\vskip -0.05in
\caption{\label{alg:transition_generic} Get 
transition distribution  $\mby_s|x$}
\end{algorithm}

\subsection{Transitions with $Q,D$}

Current param matrices $\tilde{\mbQ}_\phi,\tilde{\mbD}_\phi$ and
along  with fixed time-in scalar-out functions $b_q(s),b_d(s)$ and their known integrals
    $B_q(s),B_d(s)$.  $q_\gamma(v_0|z_0=x)$ taken
    to be parameterless so that $v_0 \sim \mathcal{N}(0,I)$. Model params are $s_\theta$ fixed $\pi_\theta$.

 \begin{algorithm}[h!]
\begin{algorithmic}
 \STATE {\bfseries Input:}
 time $s$ and current params $\phi$
 
 \STATE {\bfseries compute:}
    $[b_q]_s=\int_0^s b_q(\nu)d\nu$ using known integral
$B_q(s)-B_q(0)$

\STATE {\bfseries compute:}
    $[b_d]_s=\int_0^s b_d(\nu) d\nu$
using known integral
$B_d(s)-B_d(0)$.
    \STATE {\bfseries compute:}
    $[\mbQ_\phi]_s=[b_q]_s \cdot \Big[\tilde{\mbQ}_\phi  -\tilde{\mbQ}_\phi^\top\Big]$
for current params $\tilde{\mbQ}_\phi$.
\STATE {\bfseries compute:}
$[\mbD_\phi]_s=[b_d]_s \cdot \Big[\tilde{\mbD}_\phi
    \tilde{\mbD}_\phi^\top \Big]$
    for current params
    $\tilde{\mbD}_\phi$.
 
\STATE {\bfseries compute:}
    $\mbM_s = - ([\mbQ_\phi]_s+[\mbD_\phi]_s)$
    ($\mbM$ just a variable name)
\STATE {\bfseries compute:}
    $\mbN_s = [2\mbD_\phi]_s = 2 \cdot [\mbD_\phi]_s$
    ($\mbN$ just a variable name)
\STATE {\bfseries compute:} $\mbQ_s = b_q(s) \cdot
   \Big[\tilde{\mbQ}_\phi 
    -\tilde{\mbQ}_\phi^\top\Big]$ (not integrated)
\STATE {\bfseries compute:} $\mbD_s = b_d(s) \cdot 
    \Big[\tilde{\mbD}_\phi
    \tilde{\mbD}_\phi^\top \Big]$
    (not integrated)

\STATE {\bfseries compute:} $A_s= -[\mbQ_s + \mbD_s] $ (drift coef.)
\STATE {\bfseries compute:} $g_s^2 = 2 \mbD_s$ (diffusion coef. squared)

\STATE{ \bfseries Output:}
    $\mbA_s,g^2_s,
    \mbM_s,\mbN_s$

\end{algorithmic}
\vskip -0.05in
\caption{\label{alg:get_qd} Get $\mbQ,\mbD$ and their integrated terms
$\mbM,\mbN$}
\end{algorithm}

\begin{algorithm}[h!]
\begin{algorithmic}
 \STATE {\bfseries Input:}
    Sample $\mby_0=(x,v)$ and time $s$. Current params $\phi$
    
\STATE {\bfseries set:}
 $\mbA_s,g^2_s,
    \mbM_s,\mbN_s
    \leftarrow$ \cref{alg:get_qd}
   
\STATE {\bfseries compute:} $\mbm_{s|0}=\exp
    \Big( 
    \mbM_s
    \Big) \mby_0$ (transition mean)
\STATE {\bfseries compute:} ingredients for transition cov. matrix:
        \begin{align}
            \begin{pmatrix}
        \mbC_s \\
        \mbH_s
      \end{pmatrix} =
      \exp\Bigg[
      \begin{pmatrix}
        \mbM_s & \mbN_s \\
        \mathbold{0} &  -\mbM_s^\top 
      \end{pmatrix}
      \Bigg]
            \begin{pmatrix}
              \mathbold{0} \\
              \mathbold{I} 
            \end{pmatrix} 
        \end{align}
\STATE {\bfseries compute:} $\mathbold{\Sigma}_{s|0}=\mbC_s({\mbH_s})^{-1}$ (transition cov).
\STATE {\bfseries instantiate:} $q_{\phi,s,(x,v)} = q_\phi(\mby_s|\mby_0) = \mathcal{N}(\mbm_{s|0},\boldsymbol{\Sigma}_{s|0})$.
\STATE{ \bfseries Output:} $q_{\phi,s,(x,v)}, A_s, g_s^2$
\end{algorithmic}
\vskip -0.05in
\caption{\label{alg:transition} Get transition distributions}
\end{algorithm}

\subsection{elbo algorithms}

\begin{algorithm}[h!]
\begin{algorithmic}
 \STATE {\bfseries input:}
    Data point $x$ and current params $\theta,\phi,\gamma$
     \STATE {\bfseries draw:} an aux. sample $v \sim q_\gamma(v|x)$
     \STATE {\bfseries draw:} a sample $s \sim \text{Unif}(0,T)$
     \STATE {\bfseries set:} 
     $\mby_0=(x,v)$
     \STATE {\bfseries set:} 
     $q_{\phi,s,\mby_0}, A_s, g_s^2 \leftarrow$
     \cref{alg:transition} called on $\mby_0, s, \phi$
     \STATE {\bfseries draw:} $\mby_{s} \sim q_{\phi,s,\mby_0}$
      \STATE {\bfseries compute:}
      $\ell_{s}$ with $\text{dsm}(s)$ (\cref{alg:dsm}) or $\text{ism}(s)$
      (\cref{alg:ism})
      on $\mby_{s}, \theta, A_s, g_s^2,
      q_{\phi,s,\mby_0}$
      \STATE {\bfseries set:} 
      $q_{\phi,T,\mby_0}, \_\_, \_\_ \leftarrow$
      \cref{alg:transition} called on $\mby_0, T, \phi$
     \STATE {\bfseries draw:}  $\mby_{T} \sim q_{\phi,T,\mby_0}$
    \STATE{ \bfseries output:
$\ell_s + \log \pi_\theta(\mby_T) - \log q_\gamma(v)$}
\end{algorithmic}
\vskip -0.05in
\caption{\label{alg:elbo}Compute \gls{elbo}  with ism or dsm}
\end{algorithm}

\begin{algorithm}[h!]
\begin{algorithmic}
 \STATE {\bfseries input:}
 $\mby_s$, $\theta$, $A_s$, $g_s^2$, $q_{\phi,s,\mby_0}$.

     \STATE {\bfseries compute:}
     $\text{fwd-score}=\nabla_{\mby_s}  \log q_\phi(\mby_s|\mby_0)$
     
     \STATE {\bfseries compute:}
 $\text{model-score}=s_\theta(\mby_s,s)$
     
      \STATE {\bfseries compute:}
      $\text{fwd-score-term}=\frac{1}{2}(\text{fwd-score})^\top g_s^2 (\text{fwd-score})$
      
      \STATE {\bfseries compute:}
      $\text{score-diff}=\text{model-score} -  \text{fwd-score}$
      
    \STATE {\bfseries compute:}
    $\text{diff-term}=
    -\frac{1}{2} 
    \text{score-diff}^\top 
    g_s^2
    \text{score-diff}
    $
    \STATE {\bfseries compute:}
    $\text{div-f}= \nabla_{\mby_s} \cdot A_s \mby_s$

    \STATE{ \bfseries output:}
   $ \text{dsm}(s) = \text{fwd-score-term} +
    \text{diff-term} + \text{div-f}$

\end{algorithmic}
\vskip -0.05in
\caption{\label{alg:dsm}Compute $\text{dsm}(s)$}
\end{algorithm}

\begin{algorithm}[h!]
\begin{algorithmic}
 \STATE {\bfseries input:}
 $\mby_s$, $\theta$, $A_s$, $g_s^2$, $q_{\phi,s,\mby_0}$.

    \STATE {\bfseries compute:}
 $\text{model-score}=s_\theta(\mby_s,s)$
 
  \STATE {\bfseries compute:}
  $\text{score-term}=-\frac{1}{2}
    \text{model-score}^\top g_s^2 \text{model-score}$
    
     \STATE {\bfseries compute:}
     $\text{div-gs}
     =
     \nabla_{\mby_s} \cdot
     g_s^2 s_\theta(\mby_s,s)
     $
 \STATE {\bfseries compute:}
    $\text{div-f}= \nabla_{\mby_s} \cdot A_s \mby_s$
     \STATE {\bfseries compute:}
     $\text{div-term} = 
     - \text{div-gs}
     +
     \text{div-f}$
     
     \STATE{ \bfseries output:}
      $\text{ism}(s)
      =
      \text{score-term}
      +
      \text{div-term}
      $
\end{algorithmic}
\vskip -0.05in
\caption{\label{alg:ism}Compute $\text{ism}(s)$}
\end{algorithm}

\newpage 

\section{Valid \gls{elbo} with truncation}\label{appsec:offset_math}
The integrand in the \gls{elbo} and its gradients is not bounded at time $0$. Therefore, following \cite{sohl2015deep} and \cite{song2021maximum} the integrand in \cref{eq:mdm_bound} is integrated from $[\epsilon, T]$, rather than $[0, T]$. However, that integral is not a valid lower bound on $\log p_\theta(x)$. 
Instead, it can be viewed as a proper lower bound on the prior for a latent variable $\mby_\epsilon$. Therefore, to provide a bound for the data, one can introduce a likelihood and substitute the prior lower bound into a standard variational bound that integrates out the latent.

To provide a valid lower bound for multivariate diffusions, we extend theorem 6 in \citet{song2021maximum} from univariate to multivariate diffusions. 

\begin{theorem}\label{thm:proper_elbo}
    For transition kernel $q_\phi(\rvy_s \mid \rvy_0)$, we can compute upper bound the model likelihood at time $0$ as follows, for any $\epsilon > 0$
    \begin{align}
        \log p_\theta(x) \geq \E_{q_\phi(\rvy^v_0 \mid x)} \E_{q_\phi(\rvy_\epsilon | \rvy_0)} \left[ \log \frac{p_\theta(\rvy_0 \mid \rvy_\epsilon)}{q_\phi(\rvy_\epsilon \mid \rvy_0)} 
 + \mathcal{L}_{\text{mdm}}(\rvy_\epsilon, \epsilon) - \log q_{\phi}(\rvy_0^v \mid x)  \right] , 
    \end{align}
    where $\mathcal{L}_{\text{mdm}}(\rvy_\epsilon, \epsilon)$ is defined as
    \begin{align*}
        \mathcal{L}_{\text{mdm}}(\rvy_\epsilon, \epsilon) = \E_{q_\phi(\rvy_{>\epsilon} \mid \rvy_\epsilon)} \left[ \log \pi_\theta(\rvy_T)  - \int_{\epsilon}^T \frac{1}{2} \norm{s_\phi}_{g_{\phi}}^2 - \frac{1}{2} \norm{s_\theta - s_\phi}^2_{g_\phi} + \nabla \cdot f_\phi \right]  .
    \end{align*}    
\end{theorem}

\begin{proof}
    For transition kernel $q_\phi(\rvy_s \mid \rvy_0)$, we can compute upper bound the model likelihood at time $0$ following an application of the variational bound
    \begin{align*}
        \log p_\theta(x) &= \log \int_{v_0} p_\theta(\rvy_0=[x, v_0]) dv_0 \\
        &= \log \int_{v_0, \rvy_\epsilon} p_\theta(\rvy_0, \rvy_\epsilon) dv_0 d\rvy_\epsilon \\
        &= \log \int_{v_0, \rvy_\epsilon} q_\phi(\rvy_\epsilon \mid \rvy_0) \frac{q(v_0 \mid x)}{q(v_0 \mid x)}\frac{p_\theta(\rvy_0, \rvy_\epsilon)}{q_\phi(\rvy_\epsilon \mid \rvy_0)} dv_0 d\rvy_\epsilon \\
        &= \log \int_{v_0, \rvy_\epsilon} q_\phi(\rvy_\epsilon \mid \rvy_0) \frac{q(v_0 \mid x)}{q(v_0 \mid x)}\frac{p_\theta(\rvy_0 \mid \rvy_\epsilon) p_\theta(\rvy_\epsilon)}{q_\phi(\rvy_\epsilon \mid \rvy_0)} dv_0 d\rvy_\epsilon \\ 
        & \geq \E_{q(v_0 \mid x) q_\phi(\rvy_\epsilon \mid \rvy_0)} \left[ \log \frac{p_\theta(\rvy_0 \mid \rvy_\epsilon)}{q_\phi(\rvy_\epsilon \mid \rvy_0)} - \log q_{\phi}(\rvy_0^v \mid x) + \log p_\theta(\rvy_\epsilon) \right]
    \end{align*}

   A lower bound for $\log p_\theta(\rvy_\epsilon)$ can be derived in a similar manner to \cref{eq:mdm_bound}, such that
    \begin{align*}
       \log p_\theta(\rvy_\epsilon) \geq \mathcal{L}_{\text{mdm}}(\rvy_\epsilon, \epsilon) = \E_{q_\phi(\rvy_{>\epsilon} \mid \rvy_\epsilon)} \left[ \log \pi_\theta(\rvy_T)  - \int_{\epsilon}^T \frac{1}{2} \norm{s_\phi}_{g_{\phi}}^2 - \frac{1}{2} \norm{s_\theta - s_\phi}^2_{g_\phi} + \nabla \cdot f_\phi \right] .
    \end{align*}
    The choice of $p_\theta(\rvy_0 \mid \rvy_\epsilon)$ is arbitrary, however following \citet{sohl2015deep, song2021maximum} we let $p_\theta(\rvy_0 \mid \rvy_\epsilon)$ be Gaussian with mean $\mu_{p_\theta,\epsilon}$ and covariance $\Sigma_{p_\theta,\epsilon}$. Suppose $q_\phi(\rvy_\epsilon \mid \rvy_0) = \gN(\rvy_\epsilon \mid 
    \rmA \rvy_0, \Sigma)$, then we select the following mean $\mu_{p_\theta,\epsilon}$ and covariance $\Sigma_{p_\theta,\epsilon}$ for $p_\theta(\rvy_0 \mid \rvy_\epsilon)$
    \begin{align*}
        \mu_{p_\theta,\epsilon} &= \rmA^{-1} \Sigma s_\theta(\rvy_\epsilon, \epsilon) + \rmA^{-1} \rvy_\epsilon \\
        \Sigma_{p_\theta,\epsilon} &= \rmA^{-1} \Sigma \rmA^{-\top} 
    \end{align*}
    where $\mu_{p_\theta,\epsilon}, \Sigma_{p_\theta,\epsilon}$ are derived using Tweedie's formula \citep{efron2011tweedie} by setting $\mu_\epsilon = \E[\rvy_0 \mid \rvy_\epsilon]$ and $\Sigma_\epsilon = \var{\rvy_0 \mid \rvy_\epsilon}$. 
\end{proof}
We next derive this choice as an approximation of the optimal Gaussian likelihood.

\subsection{Likelihood derivation}

Suppose  $\rvy_0 \sim q_0(\rvy_0)$ and $\rvy_\epsilon \sim \gN(\rvy_\epsilon \mid A \rvy_0, \Sigma)$. Here, $A, \Sigma$ are the mean coefficient and covariance derived from the transition kernel at time $\epsilon$. We use Tweedie's formula to get
the mean and covariance of $\mby_0$ given $\mby_\epsilon$ under $q$.
This mean and covariance feature the true
score $\nabla_{\mby_\epsilon} \log q(\mby_{\epsilon})$.
We replace the score with the score model
$s_\theta$ and then set $p_\theta(\mby_0|\mby_\epsilon)$ to have 
the resulting approximate mean and covariance.
We make this choice because the optimal $p_\theta(\mby_0 | \mby_{\epsilon})$
equals the true
$q(\mby_0|\mby_{\epsilon})$ as discussed throughout the work.

Here $\rvy_0 = [\rvx_0, \rvv_0]$ where $\rvx_0 \sim q_{\text{data}}$.

Let $\eta$ be the natural parameter for the multivariate Gaussian likelihood $\gN(\rvy_\epsilon \mid A \rvy_0, \Sigma)$. Then, Tweedie's formula \citep{efron2011tweedie} states that:
\begin{align*}
    \E [\eta \mid \rvu_{\epsilon}] = \nabla_{\rvy_\epsilon} l(\rvy_\epsilon) - \nabla_{\rvy_{\epsilon}} l_0(\rvy_{\epsilon})
\end{align*}
\begin{itemize}

 \item 
   $ l(\rvy_\epsilon) = \log q(\rvy_\epsilon)$
 
    \item $s_\theta(\mby_\epsilon,\epsilon)$ is taken to be the true score $\nabla_{\mby_{\epsilon}} \log q(\mby_\epsilon)$ so that 
    $\nabla_{\rvy_\epsilon} l(\rvy_\epsilon) =s_{\theta}(\rvy_\epsilon, \epsilon)$

    \item $l_0$ is the log of the base distribution defined in the exponential family parameterization. 
\end{itemize}

The base distribution is a multivariate Gaussian with mean $0$ and covariance $\Sigma$, therefore $\nabla_{\rvy_{\epsilon}} l_0 (\rvy_\epsilon) = -\Sigma^{-1}\rvy_{\epsilon}$, 
\begin{align*}
    \E [\eta \mid \rvy_{\epsilon}] =s_{\theta}(\rvy_\epsilon, \epsilon)  + 
    \Sigma^{-1} \rvy_{\epsilon}.
\end{align*}
However, Tweedie's formula is not directly applicable 
since our $\mby_\epsilon$ is not directly normal with mean $\mby_0$. Instead, to derive the conditional mean of $\rvy_0$ given $\rvy_\epsilon$, we use the relation $\eta = \Sigma^{-1}\rmA \rvy_0$ and the linearity of conditional expectation to get
\begin{align*}
    \E [\rvy_0 \mid \rvy_{\epsilon}] &= \E[A^{-1}\Sigma \eta | \rvy_{\epsilon}] \\
    &= A^{-1}\Sigma \E [\eta \mid \rvy_{\epsilon}] \\
    &= A^{-1}\Sigma \left(s_{\theta}(\rvy_\epsilon, \epsilon)  + 
    \Sigma^{-1} \rvy_{\epsilon} \right) \\
    &= A^{-1}\Sigma s_{\theta}(\rvy_\epsilon, \epsilon) + A^{-1} \rvy_{\epsilon} .
\end{align*}

For the variance, we use the following relation $\rvy_\epsilon = A\rvy_0 + \sqrt{\Sigma} \epsilon$, which implies that 
\begin{align*}
    \rvy_0 &= A^{-1} \rvy_\epsilon - A^{-1} \sqrt{\Sigma} \epsilon \\
    \var{\rvy_0 \mid \rvy_{\epsilon}} &= A^{-1} \Sigma A^{-T} .
\end{align*}

Therefore, for the model posterior distribution $p_\theta(\rvy_0 \mid \rvy_\epsilon)$ we choose a Normal with mean and covariance
\begin{align*}
    \mu_{p_\theta, \epsilon} &= A^{-1}\Sigma s_{\theta}(\rvy_\epsilon, \epsilon) + A^{-1} \rvy_{\epsilon} \\
    \Sigma_{p_\theta,\epsilon} &= A^{-1} \Sigma A^{-T} 
\end{align*}
\end{document}